\documentclass[a4paper]{article}

\usepackage{amssymb}
\usepackage{amsmath,amsthm}
\usepackage{enumerate}
\usepackage{xcolor}
\usepackage{soul}
\usepackage{tikz}
\usepackage{url}
\usepackage{indentfirst}
\usepackage[subnum]{cases}
\usepackage{bm}
\usepackage[numbers,sort&compress]{natbib}
\usepackage[titlenumbered,ruled,lined,linesnumbered,algo2e]{algorithm2e}
\usepackage{subfigure}
\usepackage{multirow}
\usepackage{hyperref}

\newtheorem{theorem}{Theorem}
\newtheorem{lemma}[theorem]{Lemma}
\newtheorem{proposition}[theorem]{Proposition}

\theoremstyle{definition}
\newtheorem{definition}{Definition}

\newcommand{\nbb}{\mathbb{N}}
\newcommand{\fcal}{\mathcal{F}}
\newcommand{\xcal}{\mathcal{X}}
\newcommand{\wcal}{\mathcal{W}}
\newcommand{\ccal}{\mathcal{C}}

\newcommand{\zcal}{\mathcal{Z}}

\newcommand{\sgn}{\mbox{sgn}}

\newcommand{\ycal}{\mathcal{Y}}

\newcommand{\ebb}{\mathbb{E}}

\newcommand{\rbb}{\mathbb{R}}
\newcommand{\inn}[1]{\langle#1\rangle}

\numberwithin{equation}{section}

\title{Convergence of Online Mirror Descent\footnote{Published in Applied and Computational Harmonic Analysis, 2020} }
\author{Yunwen Lei\thanks{Y. Lei is with Department of Mathematics, City University of Hong Kong, Kowloon, Hong Kong, China (e-mail: yunwelei@cityu.edu.hk).},\;
        and
        Ding-Xuan Zhou\thanks{D.-X. Zhou is with Department of Mathematics, City University of Hong Kong, Kowloon, Hong Kong, China (e-mail: mazhou@cityu.edu.hk).}\;   
        }
\date{}
\begin{document}
\maketitle
\begin{abstract}

In this paper we consider online mirror descent (OMD) algorithms, a class of scalable online learning algorithms exploiting data geometric structures through mirror maps.
Necessary and sufficient conditions are presented in terms of the step size sequence $\{\eta_t\}_{t}$ for the convergence of an OMD algorithm with respect to the expected
Bregman distance induced by the mirror map. The condition is $\lim_{t\to\infty}\eta_t=0, \sum_{t=1}^{\infty}\eta_t=\infty$ in the case of positive variances.
It is reduced to $\sum_{t=1}^{\infty}\eta_t=\infty$ in the case of zero variances for which the linear convergence may be achieved by taking a constant step size sequence.
A sufficient condition on the almost sure convergence is also given. We establish tight error bounds under mild conditions on the mirror map, the loss function, and the regularizer.
Our results are achieved by some novel analysis on the one-step progress of the OMD algorithm
using smoothness and strong convexity of the mirror map and the loss function.

\end{abstract}

\section{Introduction}

Analyzing and processing big data in various applications has raised the need of scalable learning algorithms using geometric structures of data. One approach for scalability in learning theory is stochastic gradient descent and online learning. In this paper we are interested in online mirror descent algorithms, a class of scalable learning algorithms exploiting possible data geometric structures such as sparsity.

Mirror descent is a powerful extension of the classical gradient descent method~\citep{beck2003mirror}
by relaxing the Hilbert space structure and using a mirror map $\Psi:\wcal\to\rbb$ to capture geometric properties of data from a Banach space $\wcal$.
In this paper we consider $\wcal=\rbb^d$ endowed with a norm $\|\cdot\|$ which might be a non-Euclidean norm, allowing us to capture non-Euclidean geometric structures of data from $\rbb^d$.
To introduce the mirror descent and online mirror descent algorithms, we assume that the mirror map $\Psi$ is Fr\'{e}chet differentiable and strongly convex.
The Fr\'{e}chet differentiability means the existence of a bounded linear operator $\nabla\Psi(w): \wcal\to\rbb$ at every $w\in \wcal$ satisfying $\Psi(w+x)-\Psi(w)-\nabla\Psi(w)x =o(\|x\|)$. The strong convexity of
$\Psi$ means the existence of some $\sigma_\Psi>0$ such that
$$
   D_{\Psi}(\tilde{w}, w):=\Psi(\tilde{w})-\Psi(w)-\inn{\tilde{w}-w,\nabla\Psi(w)}\geq\frac{\sigma_\Psi}{2}\|\tilde{w}-w\|^2,\quad\forall \tilde{w}, w\in\wcal,
$$
where $\inn{\tilde{w}-w,\nabla\Psi(w)}$ is the linear operator $\nabla\Psi(w)$ acting on $\tilde{w}-w\in\wcal$. With this number $\sigma_\Psi$, we say $\Psi$ is $\sigma_\Psi$-strongly convex (with respect to the norm $\|\cdot\|$), which we assume throughout the paper.
The quantity $D_{\Psi}(\tilde{w}, w)$ is called the Bregman distance between $\tilde{w}$ and $w$.

Given a differentiable and convex objective function $F:\wcal\to\rbb$, a mirror descent algorithm approximates a minimizer of $F$ by a sequence $\{w_t\}_{t\in \nbb}\subset \wcal$ defined with an initial vector $w_1 \in \wcal$
and the gradient descent method in terms of the gradient $\nabla F$ of $F$ as
\begin{equation}\label{onlineF}
  \nabla\Psi(w_{t+1})=\nabla\Psi(w_t)-\eta_t\nabla F(w_t),\qquad t\in\nbb,
\end{equation}
where $\{\eta_t\}$ is a sequence of positive numbers called the step size sequence.
Here the gradient descent is performed in the dual $(\wcal^* =\rbb^d,\|\cdot\|_*)$ of the primal space $(\wcal, \|\cdot\|)$ since
the map $\nabla\Psi: \wcal\to\wcal^*$ is well-defined, and invertible due to the strong convexity of $\Psi$. Useful instantiations~\citep{duchi2010composite} of the mirror map $\Psi$
include the choice of {\bf $p$-norm divergence} $\Psi =\Psi_p$ with $1<p\leq 2$ defined by $\Psi_p(w)=\frac{1}{2}\|w\|_p^2$ where
$\|\cdot\|_p$ is the $p$-norm defined by $\|w\|_p=\left(\sum_{i=1}^{d}|w(i)|^p\right)^{1/p}$ for $w=(w(1), \ldots, w(d)) \in \rbb^d$.
The mirror descent algorithm with $\Psi=\Psi_2$ recovers the gradient descent algorithm.

In machine learning, the objective function $F$ is often the regularized risk
$F(w)=\ebb_Z [f(w,Z)]$ of the linear function $x\to\inn{w,x}$ induced by the action of $x\in \wcal^*$ on $w\in\wcal$, where $f(w,Z)=\phi(\inn{w,X}, Y)+r(w)$ is the regularized loss function induced
by a loss function $\phi:\rbb\times\rbb\to\rbb_+$ and a convex regularizer $r:\wcal\to\rbb_+$, and $\ebb_Z$ denotes the expectation with respect to the random sample $Z=(X,Y)$
drawn from a Borel probability measure $\rho$ on $\zcal:=\xcal\times\ycal$ with an input space $\xcal \subset \wcal^*$ and an output space $\ycal \subset \rbb$.

In many machine learning applications, training examples $\{z_t=(x_t,y_t)\in\zcal\}_t$ become available in a sequential manner. In such situations,
instead of computing $F(w)$, we use the sample $z_t$ at the $t$-th iteration of the mirror descent to compute the gradient $\nabla_w [f(w_t, z_t)]$ of $f(w, z_t)$ with respect to the variable $w$ at $w_t$. This leads to the
{\bf online mirror descent} (OMD) algorithm which extends the classical online gradient descent algorithm by replacing $\Psi_2$ with a mirror map $\Psi$ to capture data geometric structures beyond Hilbert spaces.
It generates a sequence $\{w_t\}_t \subset \wcal$ with an initial vector $w_1 \in \wcal$ by performing the stochastic mirror descent in the dual space as
\begin{equation}\label{online-mirror-descent}
  \nabla\Psi(w_{t+1})=\nabla\Psi(w_t)-\eta_t\nabla_w [f(w_t, z_t)], \qquad t\in\nbb.
\end{equation}
We always assume that the loss function $\phi$ is convex and differentiable with respect to the first variable (with the partial derivative $\phi'$).
When $\Psi=\Psi_2$ and $r(w)=\lambda\|w\|_2^2$ with $\lambda \geq 0$,
the OMD \eqref{online-mirror-descent} becomes the classical online learning algorithm with the iteration $w_{t+1}=w_t- \eta_t [\phi'(\inn{w_t, x_t}, y_t) x_t + 2\lambda w_t]$
generated by the stochastic gradient descent method in the Hilbert space $\wcal^* = \wcal$. The special choice $\phi(a, y) = \frac{1}{2} (a-y)^2$ of the unregularized least squares loss function with $r =0$
corresponds to the general randomized Kaczmarz algorithm \citep{chen2012almost} given by
\begin{equation}\label{kaczmarz}
  w_{t+1}=w_t - \eta_t[\inn{w_t,x_t} - y_t]x_t, \qquad t\in\nbb.
\end{equation}
It was shown in \citep{lin2015learning} that when $\inf_{w\in\wcal} \ebb_Z \left[\left(Y -\inn{w,X}\right)^2\right]>0$,
the randomized Kaczmarz algorithm (\ref{kaczmarz}) converges if and only if $\lim_{t\to\infty}\eta_t=0$ and $\sum_{t=1}^{\infty}\eta_t=\infty$.

This paper presents {\bf necessary and sufficient conditions} for the convergence of the OMD algorithm \eqref{online-mirror-descent} with respect to the {\bf Bregman distance} $D_\Psi$.
It extends the result in \citep{lin2015learning, SV09} from $\Psi_2$ to a general mirror map $\Psi$ beyond the Hilbert space framework.
Our conditions are stated in terms of the step size sequence $\{\eta_t\}_t$,
under some mild assumptions on the mirror map $\Psi$, the regularized loss function $f$, and the probability measure $\rho$.
Throughout the paper, we assume that the training examples $\{z_t\}_t$ are sampled independently from the probability measure $\rho$ on $\zcal$.

We illustrate our main results to be stated in the next section by presenting an example corresponding to the special choice
of the unregularized least squares loss and a strongly smooth mirror map
or the $p$-norm divergence $\Psi_p$ (which, as shown in Proposition \ref{prop:p-divergence-nonsmooth}, is not strongly smooth).
Here we say that $\Psi$ is $L_\Psi$-strongly smooth (with respect to the norm $\|\cdot\|$) with $L_\Psi >0$ if $D_{\Psi}(\tilde{w}, w) \leq \frac{L_\Psi}{2}\|\tilde{w} -w\|^2$ for any $w,\tilde{w}\in\wcal$.
Examples of strongly smooth mirror maps include $\Psi_2$ and a mirror map $\Psi^{(\epsilon, \lambda)}$ with parameters $\epsilon>0, \lambda>0$
defined in the literature of compressed sensing ~\citep{cai2009linearized} as
$\Psi^{(\epsilon, \lambda)}(w)=\lambda \sum_{i=1}^{d} g_\epsilon(w(i)) +\frac{1}{2}\|w\|_2^2$,
where $g_\epsilon(\xi)=\frac{\xi^2}{2\epsilon}$ for $|\xi|\leq\epsilon$ and $|\xi|-\frac{\epsilon}{2}$ for $|\xi|>\epsilon$.
The mirror map $\Psi_p$ plays an important role in the mirror descent method
and the specific choice with $p = 1 + \frac{1}{\log d}$ gives convergence bounds
with a logarithmic dependence on the dimension $d$, see \citep{duchi2010composite}.
It is strongly convex with $\sigma_{\Psi_p} =p-1$ when the norm of $\wcal$ takes the $p$-norm $\|\cdot\| =\|\cdot\|_p$ (see~\citep{ball1994norm}), and
by the norm equivalence, $\sigma_{\Psi_p}>0$ for other norms.

With the special choice of the unregularized least squares loss $f(w,z)=\frac{1}{2} (y -\inn{w,x})^2$, the OMD algorithm (\ref{online-mirror-descent})
takes a special form
\begin{equation}\label{lsOMD}
  \nabla\Psi(w_{t+1})=\nabla\Psi(w_t) - \eta_t[\inn{w_t,x_t} - y_t]x_t, \qquad t\in\nbb.
\end{equation}
The following result for this example will be proved in Section \ref{sec:proof-incremental-convex}. Denote by $X^\top$ the transpose of $X\in\wcal^*$.

\begin{theorem}\label{thm:least-square}
Assume $\sup_{x\in\xcal}\|x\|_* <\infty$, $\ebb_Z [Y^2] <\infty$, and that the covariance matrix $\ccal_X=\ebb_Z [XX^\top]$ is positive definite.
Consider the OMD algorithm (\ref{lsOMD}) and denote $w_\rho =\ccal_X^{-1}\ebb_Z [XY]$.
Let $\Psi$ be either some $p$-norm divergence $\Psi=\Psi_p$ with $1<p\leq2$ or a strongly smooth mirror map.
\begin{enumerate}[(a)]
    \item Assume $\inf_{w\in\wcal} \ebb_Z \left[\left|Y -\inn{w,X}\right| \|X\|_*\right]>0$. Then $\lim_{t\to\infty}\ebb_{z_1, \ldots, z_{t-1}} [\|w_\rho -  w_t\|^2]=0$ if and only if
    \begin{equation}\label{nece-suff}
      \lim_{t\to\infty}\eta_t=0\quad\text{and}\quad\sum_{t=1}^{\infty}\eta_t=\infty.
    \end{equation}
     Furthermore, if $\Psi$ is strongly smooth and $\lim_{t\to\infty}\eta_t=0$, then there exist some $\tilde{T_1}\in\nbb$ and $\tilde{C} >0$
     such that $\ebb_{z_1,\ldots,z_{T-1}}[\|w_\rho -  w_T\|^2]\geq \tilde{C}T^{-1}$ for $T\geq\tilde{T}_1$. If we take $\eta_t=\frac{4}{(t+1)\sigma}$ for some appropriate $\sigma>0$ (given in the proof), then $\ebb_{z_1,\ldots,z_{T-1}}[\|w_\rho -  w_T\|^2] =O\left(T^{-1}\right)$.
    \item Assume $w_\rho \neq w_1, \ebb_Z \left[\left|Y -\inn{w_\rho, X}\right| \|X\|_*\right]=0$ and for some $\kappa>0$, $\eta_t \leq \frac{\sigma_\Psi}{(2+\kappa)R^2}$.
    Then $\lim_{t\to\infty}\ebb_{z_1, \ldots, z_{t-1}} [\|w_\rho -  w_t\|^2]=0$ if and only if $\sum_{t=1}^{\infty}\eta_t=\infty$.
    Furthermore, if $\Psi$ is strongly smooth and $\eta_t\equiv \eta_1 <\frac{\sigma_\Psi}{2 R^2}$, then there exist $\tilde{c}_1,\tilde{c}_2\in(0,1)$ such that
        \begin{equation}\label{rate-lsquare-novariance}
          \tilde{c}_1^T \|w_\rho -  w_1\|^2 \leq \ebb_{z_1,\ldots,z_{T-1}}[\|w_\rho -  w_T\|^2]\leq \tilde{c}_2^T \|w_\rho -  w_1\|^2, \qquad \forall T\in\nbb.
        \end{equation}
    \item If the step size sequence satisfies
    \begin{equation}\label{ae-sufficient-step-condition}
      \sum_{t=1}^{\infty}\eta_t=\infty\quad\text{and}\quad\sum_{t=1}^{\infty}\eta_t^2<\infty,
    \end{equation}
    then $\{\|w_\rho -  w_t\|^2\}_{t\in\nbb}$ converges to $0$ almost surely.
  \end{enumerate}
\end{theorem}

Part (b) of Theorem \ref{thm:least-square} is for the case of zero variances with $y=\inn{w_\rho, x}$ almost surely,
meaning that the sampling process has no noise and the target function (conditional mean) is linear.
It asserts that the OMD algorithm with a strongly smooth mirror map and a constant step size sequence
may converge linearly in this case. Part (a) asserts that for the case of positive variances
(either the sampling process has noise or the target function is nonlinear)
the OMD algorithm with a strongly smooth mirror map can converge of at most order $O(\frac{1}{T})$ which is achievable.
This solves a conjecture raised in~\citep[page 3346]{lin2015learning} that a convergence rate of order $O(T^{-\theta})$ with $1<\theta\leq 2$
is impossible for the randomized Kaczmarz algorithm (with $\Psi=\Psi_2$) in the noisy case.
Theorem \ref{thm:least-square} also characterizes the convergence in expectation
by means of the step size condition $\sum_{t=1}^{\infty}\eta_t=\infty$ for the case of zero variances and the condition
$\lim_{t\to\infty}\eta_t=0$ and $\sum_{t=1}^{\infty}\eta_t=\infty$ for the case of positive variances.

Our analysis is based on a key identity on measuring the one-step progress of the OMD algorithm by excess Bregman distances,
from which lower and upper bounds on the one-step progress are established by using strong smoothness and convexity
of the associated regularized loss functions as well as properties of the mirror map. These lower and upper bounds are then used to build necessary and sufficient conditions, as well as tight convergence rates.

\section{Main Results\label{sec:main-result}}

In this section we state our main results on necessary and sufficient conditions for the convergence of the OMD algorithm \eqref{online-mirror-descent}
to a minimizer $w^* =\arg \min_{w\in\wcal}F(w)$ of the regularized risk $F$ which is assumed to exist throughout the paper.

Our discussion requires some mild assumptions on the mirror map $\Psi$ and the regularized risk $F$.
On the mirror map, for necessary conditions, we shall assume that $\nabla \Psi$ is continuous at $w^*$ and satisfies the following incremental condition at infinity.

\begin{definition}
We say that $\nabla \Psi$ satisfies an incremental condition (of order $1$) at infinity if there exists a constant $C_\Psi>0$ such that
\begin{equation}\label{IncrePsi}
\|\nabla \Psi (w)\|_* \leq C_\Psi (1 +\|w\|), \qquad \forall w \in \wcal.
\end{equation}
\end{definition}

We shall show later that the $p$-norm divergence $\Psi_p$ with $1<p\leq2$ and strongly smooth mirror maps satisfy this mild condition.

For the pair $(\Psi, F)$, we shall also assume the following condition measuring how the convexity of $\Psi$ is controlled by that of $F$ around $w^*$ with a convex function $\Omega$.
Recall that $w^*$ is a minimizer of $F$ on $\wcal$.

\begin{definition}
We say that the convexity of $\Psi$ is controlled by that of $F$ around $w^*$ with a convex function $\Omega: [0, \infty) \to \rbb_+$ satisfying $\Omega (0)=0$ and $\Omega (u)>0$ for $u>0$ if the pair $(\Psi, F)$ satisfies
\begin{equation}\label{convexcontrol}
\inn{w^*-w, \nabla F(w^*)-\nabla F(w)} \geq \Omega \left(D_\Psi(w^*,w)\right), \qquad \forall w\in \wcal.
\end{equation}
\end{definition}

Typical choices of the convex function $\Omega$ include $\Omega (u) = C u^\alpha$ with $\alpha \geq 1$ and $C>0$. In particular, when $F$ is strongly convex and $\Psi$ is strongly smooth,
condition (\ref{convexcontrol}) is satisfied with a linear (convex) function $\Omega (u) = C u $ for some $C>0$. To see this,
we notice from the definition of the Bregman distance that for a Fr\'{e}chet differentiable and convex function $g:\rbb^d\to\rbb$, there holds
\begin{equation}\label{Bregmansum}
D_g (w,\tilde{w}) +D_g (\tilde{w}, w) = \inn{w -\tilde{w}, \nabla g (w) -\nabla g (\tilde{w})}, \qquad \forall w,\tilde{w}\in\wcal.
\end{equation}
So when $F$ is $\sigma_F$-strongly convex with $\sigma_F >0$, we have $\inn{w^*-w, \nabla F(w^*)-\nabla F(w)} \geq \sigma_F \|w^*-w\|^2$.
It follows that (\ref{convexcontrol}) with $\Omega (u) = \frac{2\sigma_F}{L_\Psi} u$ is satisfied when $\Psi$ is $L_\Psi$-strongly smooth.

\subsection{Statements of general results}

Our first main result, Theorem \ref{thm:nece-suff}, states a necessary and sufficient condition for the convergence of the OMD algorithm for the case of positive variances
meaning that $\inf_{w\in\wcal} \ebb_Z \left[\|\nabla_w [f(w, Z)]\|_*\right]>0$.
It also states that in this case, the OMD algorithm cannot achieve convergence rates faster than $O(T^{-1})$ after $T$ iterates, while the rate $O(T^{-1})$ can be achieved
when $\Omega (u) = C u$ in (\ref{convexcontrol}). This theorem is a consequence of Propositions \ref{lem:necessary} and \ref{lem:sufficient} to be presented in Section \ref{sec:generalconds}.

\begin{theorem}\label{thm:nece-suff}
Assume $\inf_{w\in\wcal} \ebb_Z \left[\|\nabla_w [f(w, Z)]\|_*\right]>0$ and that for some constant $L>0$,
$f(\cdot, z)$ is $L$-strongly smooth for almost every $z\in Z$.
Suppose that $\nabla \Psi$ is continuous at $w^*$ and satisfies the incremental condition (\ref{IncrePsi}) at infinity, and
that the pair $(\Psi, F)$ satisfies (\ref{convexcontrol}) around $w^*$ with a convex function $\Omega: [0, \infty) \to \rbb_+$ satisfying $\Omega (0)=0$ and $\Omega (u)>0$ for $u>0$. Then for the OMD algorithm \eqref{online-mirror-descent}, $\lim_{t\to\infty} \ebb_{z_1, \ldots, z_{t-1}} [D_\Psi(w^*, w_t)]=0$ if and only if the step size sequence satisfies \eqref{nece-suff}.
  \begin{enumerate}[(a)]
    \item If $\Psi$ is strongly smooth and $\lim_{t\to\infty}\eta_t=0$, then there exist some constants $t_0\in\nbb$ and $\tilde{C}>0$ such that
  \begin{equation}\label{lower-rate}
  \ebb_{z_1, \ldots, z_{T-1}} [D_\Psi(w^*, w_T)] \geq \frac{\tilde{C}}{T-t_0+1},\qquad\forall T\geq t_0.
  \end{equation}
    \item If there exists an $\sigma_F>0$ such that
    \begin{equation}\label{strong-convexity-assumption}
     \inn{w^*-w, \nabla F(w^*)-\nabla F(w)} \geq \sigma_F D_\Psi(w^*, w), \qquad \forall w\in \wcal.
    \end{equation}
    and the step size sequence takes the form $\eta_t=\frac{4}{(t+1)\sigma_F}$,
    then
    \begin{equation}\label{one-over-T}
    \ebb_{z_1, \ldots, z_{T-1}} [D_\Psi(w^*,w_T)] =O\left(\frac{1}{T}\right).
    \end{equation}
  \end{enumerate}
\end{theorem}

We shall see from the proof of Proposition \ref{lem:necessary} given in Section \ref{sec:generalconds}
that the continuity of $\nabla \Psi$ at $w^*$ and the incremental condition (\ref{IncrePsi}) are only required
for proving $\lim_{t\to\infty} \eta_t =0$ of the necessity, they are not required for the sufficiency or for proving $\sum_{t\to\infty} \eta_t =\infty$ of the necessity.
These conditions are satisfied when $\Psi$ is strongly smooth, as shown in Proposition \ref{prop:least-squares} below.

Our second main result, Theorem \ref{thm:nece-suff-without-variance} to be proved in Section \ref{sec:zerovar}, states a necessary and sufficient condition
for the convergence of the OMD algorithm for the case of zero variances in the sense that $\ebb_Z \left[\|\nabla_w [f(w^*, Z)]\|_*\right] =0$.

\begin{theorem}\label{thm:nece-suff-without-variance}
Assume $\ebb_Z \left[\|\nabla_w [f(w^*, Z)]\|_*\right] =0$ and that for some constant $L>0$,
$f(\cdot, z)$ is $L$-strongly smooth for almost every $z\in Z$. Suppose that the pair $(\Psi, F)$ satisfies (\ref{convexcontrol}) around $w^*$ with a convex function $\Omega: [0, \infty) \to \rbb_+$
satisfying $\Omega (0)=0$ and $\Omega (u)>0$ for $u>0$. Assume also $w_1\neq w^*$ and that for some $\kappa>0$, $\eta_t \leq \frac{\sigma_\Psi}{(2+\kappa)L}$
for every $t\in\nbb$.

Then $\lim_{t\to\infty}\ebb_{z_1, \ldots, z_{t-1}} [D_\Psi(w^*, w_t)]=0$ if and only if $\sum_{t=1}^{\infty}\eta_t=\infty$.
Furthermore, if \eqref{strong-convexity-assumption} holds and $\eta_t \equiv \eta_1 <\frac{\sigma_\Psi}{2 L}$, then
  \begin{equation}\label{linear-rate}
    \left(1-2\sigma_\Psi^{-1}L\eta_1\right)^T D_\Psi(w^*,w_1)\leq \ebb_{z_1,\ldots,z_{T-1}}[D_\Psi(w^*,w_T)]\leq \left(1-2^{-1}\sigma_F\eta_1\right)^T D_\Psi(w^*,w_1).
  \end{equation}
\end{theorem}

Our last main result, Theorem \ref{thm:ae-sufficient} to be proved in Section \ref{sec:zerovar},
provides a sufficient condition for the almost sure convergence of the OMD algorithm by imposing a stronger condition with $\sum_{t=1}^\infty \eta_t^2 <\infty$.

\begin{theorem}\label{thm:ae-sufficient}
Assume that for some constant $L>0$, $f(\cdot, z)$ is $L$-strongly smooth for almost every $z\in Z$.
Suppose that the pair $(\Psi, F)$ satisfies (\ref{convexcontrol}) around $w^*$ with a convex function $\Omega: [0, \infty) \to \rbb_+$
satisfying $\Omega (0)=0$ and $\Omega (u)>0$ for $u>0$. If the step size sequence satisfies the condition \eqref{ae-sufficient-step-condition},
then we have $\lim_{t\to\infty}D_\Psi(w^*,w_t)=0$ almost surely.
\end{theorem}

\subsection{Results with strongly smooth mirror maps and $p$-norm divergence\label{sec:incremental-convex}}
In this subsection, for two classes of mirror maps $\Psi$ and strongly convex objective functions $F$, we state some results to be proved in Section \ref{sec:proof-incremental-convex} on the continuity of $\nabla \Psi$ at $w^*$ and the incremental condition (\ref{IncrePsi}) at infinity for $\nabla\Psi$, and the convexity condition \eqref{convexcontrol} of $(\Psi,F)$.

The first class of mirror maps are strongly smooth ones.

\begin{proposition}\label{prop:least-squares}
If $\Psi$ is strongly smooth, then $\nabla\Psi$ is continuous everywhere and satisfies the incremental condition \eqref{IncrePsi} at infinity. Furthermore, if $F$ is strongly convex, \eqref{convexcontrol} is satisfied for a linear convex function $\Omega (u) = C_{\Psi, L} u$ with some $C_{\Psi, L}>0$.
\end{proposition}

The second class of mirror maps are the $p$-norm divergence $\Psi=\Psi_p$ with $1<p\leq2$.
For the case $p=2$, we have $\nabla\Psi_2(w)=w$, $D_{\Psi_2}(\tilde{w}, w)=\frac{1}{2}\|w-\tilde{w}\|_2^2$ for $w,\tilde{w}\in\wcal$ and $\Psi_2$ is strongly smooth. So Proposition \ref{prop:least-squares} applies.

\begin{proposition}\label{prop:p-divergence}
Consider the $p$-norm divergence $\Psi=\Psi_p$ with $1<p <2$. Then $\nabla\Psi_p$ is continuous everywhere
and satisfies the incremental condition \eqref{IncrePsi} with $C_{\Psi_p}=1$. Moreover, we have
\begin{equation}\label{nablaPsinorms}
    \|\nabla\Psi_p(w)\|_* =\|w\|_p, \quad \forall w \in \wcal
\end{equation}
and
\begin{equation}\label{PsipBreg}
D_{\Psi_p}(\tilde{w}, w) \leq \left(\left(2 \|\tilde{w}\|_p\right)^{2-p} + \left\|\tilde{w}\right\|_p^{p-1}  + 1\right)
\left(\|\tilde{w}-w\|_p^2+ \|\tilde{w}-w\|_p^{\min\{p, 3-p\}}\right), \quad \forall \tilde{w}, w \in \wcal.
\end{equation}
Denote $\tau_p = \frac{2}{\min\{p, 3-p\}} \in (1, 2]$. For $\tilde{w} \in \wcal$, we have
\begin{equation}\label{Psipcondition}
 \|\tilde{w}-w\|_p^2 \geq B_{p} \Omega_p \left(D_{\Psi_p}(\tilde{w}, w)\right), \qquad \forall w \in \wcal,
\end{equation}
where $\Omega_p: [0, \infty) \to [0, \infty)$ is the convex function depending on $p$ defined by
\begin{equation}\label{Omegap}
\Omega_p \left(u\right) =\left\{\begin{array}{ll} u + \frac{1}{\tau_p} -1, &\hbox{if} \ u\geq 1, \\
\frac{1}{\tau_p} u^{\tau_p}, &\hbox{if} \ 0\leq u< 1, \end{array}\right.
\end{equation}
and $B_{p}$ is the constant depending on $\|\tilde{w}\|_p$ and $p$ given by
$$ B_{p} = \min\left\{\left(2\left(2 \|\tilde{w}\|_p\right)^{2-p} + 2\left\|\tilde{w}\right\|_p^{p-1}  + 2\right)^{-1}, \left(2\left(2 \|\tilde{w}\|_p\right)^{2-p} + 2\left\|\tilde{w}\right\|_p^{p-1}  + 2\right)^{-\tau_p}\right\}. $$
If $F$ is $\sigma_F$-strongly convex with respect to the norm $\|\cdot\|_p$, then the pair $(\Psi_p, F)$
satisfies \eqref{convexcontrol} around $w^*$ with the convex function $\Omega:\rbb_+\to\rbb_+$ given by
$$
\Omega(u)= \sigma_F B_{p} \Omega_p (u), \qquad u\in [0, \infty).
$$
\end{proposition}

\begin{figure}[htbp]
\centering
\includegraphics[width=0.66\textwidth]{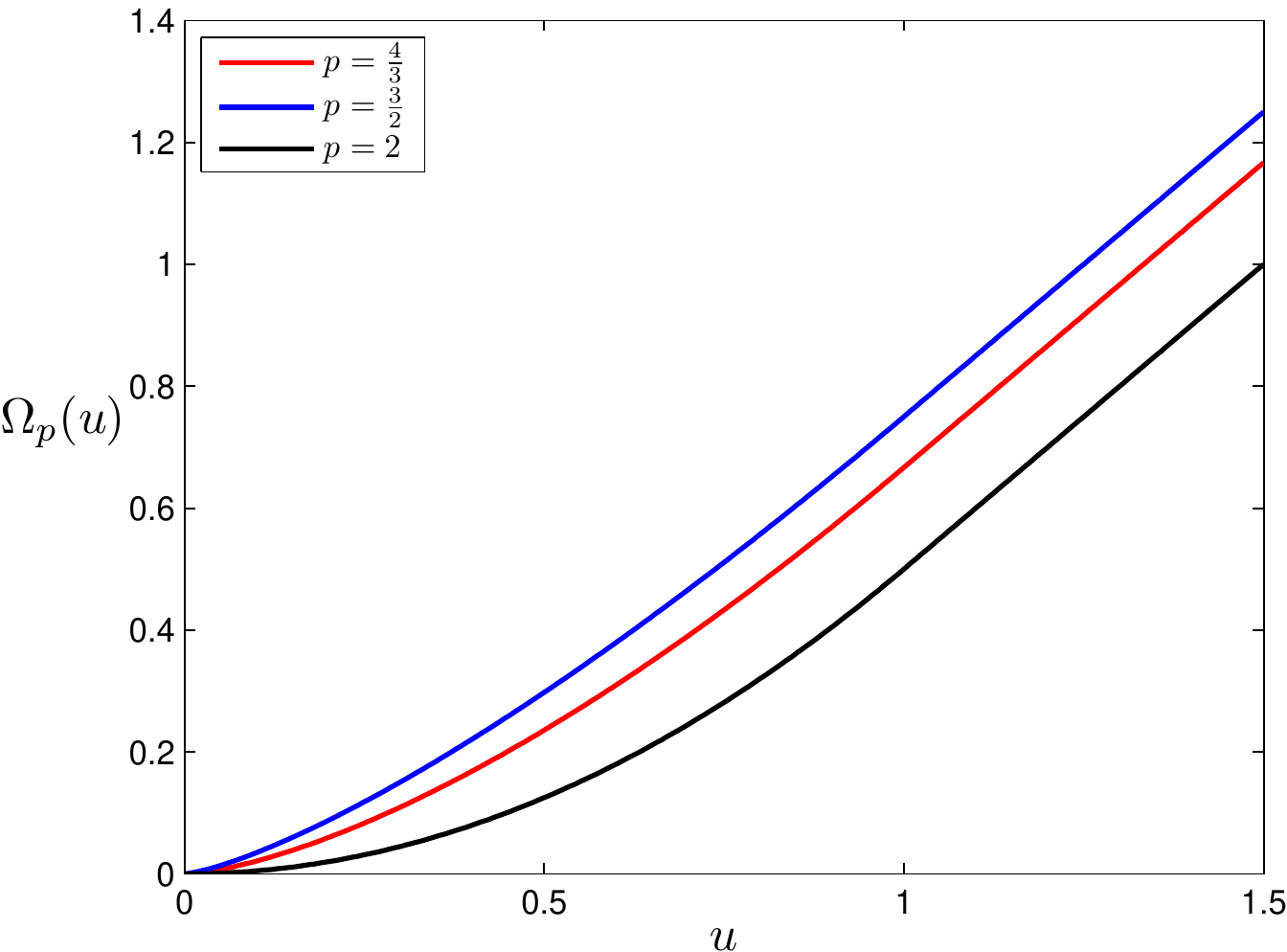}
\caption{Plots of the convex function $\Omega_p$ with $p=\frac{4}{3}$ (red line), $p=\frac{3}{2}$ (blue line) and $p=2$ (black line).\label{fig:omega}}
\end{figure}

We remark that the convex function $\Omega_2$ defined by (\ref{Omegap}) with $p=2$ is a Huber loss~\citep{huber1964}. Figure \ref{fig:omega} gives the plots of the function $\Omega_p$ with $p=\frac{4}{3},p=\frac{3}{2}$ and $p=2$.

Following Proposition \ref{prop:p-divergence}, a natural question to ask is whether the $p$-norm divergence is strongly smooth (that is, whether (\ref{Psipcondition})
holds with $\Omega_p \left(u\right) = C u$ for some $C>0$). When $d=1$, $\Psi_p (w) = \frac{1}{2} w^2 =\Psi_2 (w)$ is strongly smooth. When $d>1$, the answer is negative, as shown in the
following proposition to be proved in the appendix.

\begin{proposition}\label{prop:p-divergence-nonsmooth}
For $d>1$, the $p$-norm divergence $\Psi=\Psi_p$ with $1<p<2$ is not strongly smooth.
\end{proposition}

\subsection{Explicit results with special loss functions for learning\label{sec:loss function}}
In this subsection we state explicit results on the convergence of the OMD algorithm
associated with the regularized loss function $f(w,z)=\phi(\inn{w,x}, y)+\lambda\|w\|_2^2$ with $\lambda >0$ and the norm $\|\cdot\| = \|\cdot\|_2$ when the loss function $\phi$
has a Lipschitz continuous derivative.
Common examples of such loss functions~\citep{huber1964, chen2004support, wuyingzhou2007} include the least squares loss $\phi(a, y)=\frac{1}{2}(a-y)^2$,
the logistic loss $\phi(a, y)=\log(1+\exp(-ay))$ or $\phi(a, y)=1/(1+e^{ay})$, the $2$-norm hinge loss $\phi(a, y)=\left(\max\{0,1-ay\}\right)^2$,
and the Huber loss $\Omega_2$ defined by (\ref{Omegap}) with $p=2$.

The following explicit result will be proved in Section \ref{sec:proof-incremental-convex}.

\begin{theorem}\label{thm:p-norm-convergence}
Assume $\sup_{x\in\xcal}\|x\|_* <\infty$, $\|\cdot\|=\|\cdot\|_2$,
and the derivative $\phi'$ of the convex loss function $\phi:\rbb\times\rbb\to\rbb_+$ satisfies the Lipschitz condition
\begin{equation}\label{Lipphi}
\ell_\phi := \sup_{u \not= v\in \rbb, y\in \ycal} \frac{|\phi'(u, y) - \phi'(v, y)|}{|u-v|} <\infty.
\end{equation}
Then the regularized loss function $f(w,z)=\phi(\inn{w,x}, y)+\lambda\|w\|_2^2$ with some $\lambda>0$ is $2(\ell_\phi R^2 + \lambda)$-strongly smooth for every $z\in \zcal$. The objective function $F$ is also $2(\ell_\phi R^2 + \lambda)$-strongly smooth, and is $2\lambda$-strongly convex.
The conclusion of Theorem \ref{thm:least-square} with $w_\rho$ replaced by $w^*$ holds for the OMD algorithm \eqref{online-mirror-descent} with $\Psi$ being either some $p$-norm divergence $\Psi=\Psi_p$ with $1<p\leq2$ or a strongly smooth mirror map.
\end{theorem}

\subsection{Comparison and discussion}

In the special Hilbert space setting with $\Psi=\Psi_2$, there is a large learning theory literature on the convergence of stochastic gradient descent or online learning algorithms.
For the online gradient descent algorithm (\ref{onlineF}), under the assumption that the objective function $F$ with a single minimizer $w^*$ satisfies
$$\inf_{\|w-w^*\|_2^2>\epsilon}\inn{w-w^*,\nabla F(w)}>0, \qquad \forall \epsilon>0$$
and
$$\|\nabla F(w)\|_2^2\leq A+B\|w-w^*\|_2^2, \qquad \forall w\in\wcal $$
for some constants $A, B \geq 0$, it was shown~\citep{bottou1998online} that $\{w_t\}_t$ would converge to $w^*$ almost surely if the step sizes satisfy \eqref{ae-sufficient-step-condition}.
Convergence of online learning algorithms based on regularization schemes in reproducing kernel Hilbert spaces were discussed
in~\citep{SmaleYao, PY08} for regression and~\citep{ying2006online} for classification.
Under some assumptions on uniform boundedness of $\{w_t\}_t$ or smoothness of the loss function,
it was shown that a sufficient condition for the convergence in expectation is the step size condition \eqref{nece-suff}.
Such a result was recently established for online pairwise learning in~\citep{ying2015unregularized}.
We remark that the stochastic gradient descent method has also been well studied in the literature of optimization (see, e.g.,~\citep{robbins1951stochastic, nedic2001incremental})
under some conditions on the noise sequence instead of conditions on the step size sequence. For the randomized Kaczmarz algorithm \eqref{kaczmarz},
the convergence in expectation has been studied in the literature of non-uniform sampling and compressed sensing,
including the characterization of the convergence~\citep{lin2015learning} by \eqref{nece-suff} in the noisy case with $\inf_{w\in\wcal}\ebb_Z[(\inn{w,X}-Y)^2]>0$,
and the linear convergence~\citep{SV09} with a constant step size sequence in the noiseless case with $y= \inn{w^*, x}$ almost surely.
Our work on the convergence of the OMD algorithm \eqref{online-mirror-descent} with a general mirror map $\Psi$
is motivated by these results on the randomized Kaczmarz algorithm \eqref{kaczmarz} with the special mirror map $\Psi_2$.

For the OMD algorithm \eqref{online-mirror-descent} with a general mirror map $\Psi$, the only existing work to our best knowledge is some regret bounds in~\citep{duchi2010composite}.
In this paper we characterize the convergence in expectation by the step size condition \eqref{nece-suff} in the noisy case
and by $\sum_{t=1}^{\infty}\eta_t=\infty$ in the noiseless case, derive the linear convergence with a constant step size sequence in the noiseless case, and
verify the almost sure convergence by the step size condition (\ref{ae-sufficient-step-condition}).
The main difficulty with the general mirror map $\Psi$ is the lack of analysis for
the one-step progress $\|w_{t+1}-w^*\|_2^2-\|w_t-w^*\|_2^2$ which was carried out in~\citep{lin2015learning}
by exploiting the Hilbert space structure and the special linearity caused by the least squares loss function.
To overcome this difficulty due to the Banach space structure and the nonlinearity,
we use the Bregman distance $D_{\Psi}$ induced by the mirror map $\Psi$, which has been used in our recent work~\citep{lei2016analysisb}.
Our novelty here is a key identity (\ref{key-identity}) measuring the one-step progress of the OMD algorithm with the general mirror map $\Psi$.
Our analysis is then conducted by extensively using properties of the Bregman distance,
the smoothness and convexity of regularized loss functions, and the convexity condition (\ref{convexcontrol}) involving a related convex function $\Omega$.

Our contribution of this paper includes not only the novel convergence analysis for the OMD algorithm \eqref{online-mirror-descent} with a general mirror map $\Psi$,
but also some improvements of our earlier work~\citep{lin2015learning} on the randomized Kaczmarz algorithm \eqref{kaczmarz} with the special mirror map $\Psi_2$.
In particular, we confirm a conjecture raised in~\citep{lin2015learning} on high order convergence rates for the randomized Kaczmarz algorithm.
Furthermore, the analysis in~\citep{lin2015learning} was carried out under the restriction $0<\eta_t<2$ on the step size sequence which is removed here.
It would be interesting to get explicit convergence rates when the mirror map is $\Psi_p$, and to extend our analysis to other learning frameworks~\citep{GLZ17, HFWZ, LGZ17}.

\section{A Key Identity and Idea of Analysis\label{sec:idea}}

Our analysis for the convergence of the OMD algorithm \eqref{online-mirror-descent} will be carried out
based on the following key identity which measures the one-step progress of the algorithm in terms of the excess Bregman distance $D_\Psi(w^*,w_{t+1})-D_\Psi(w^*,w_t)$.

\begin{lemma}\label{lem:key-identity}
  The following identity holds for $t\in\nbb$
  \begin{equation}\label{key-identity}
    \ebb_{z_t}[D_\Psi(w^*,w_{t+1})]-D_\Psi(w^*,w_t)=\eta_t\inn{w^*-w_t,\nabla F(w_t)}+ \ebb_{z_t}\big[D_\Psi(w_t,w_{t+1})\big].
  \end{equation}
\end{lemma}

\begin{proof}
By the definition of the Bregman distance, we see the following identity
$$
  D_{\Psi}(w, v)+D_{\Psi}(v, u) -D_{\Psi}(w, u)=\inn{w-v, \nabla\Psi(u)-\nabla\Psi(v)}, \qquad \forall u, v, w \in \wcal.
$$
Choosing $v=w_{t+1}$ and $u=w_{t}$ yields
\begin{align*}
  D_\Psi(w,w_{t+1})-D_\Psi(w,w_t)  = -D_\Psi(w_{t+1},w_t) + \inn{w-w_{t+1},\nabla\Psi(w_t)-\nabla\Psi(w_{t+1})}.
\end{align*}
We now separate $w-w_{t+1}$ into $w-w_t$ and $w_t-w_{t+1}$, use the iteration relation \eqref{online-mirror-descent} of the OMD algorithm and apply \eqref{Bregmansum} with $g=\Psi$ to derive
\begin{align*}
       & D_\Psi(w,w_{t+1})-D_\Psi(w,w_t) \\
       &
        =-D_\Psi(w_{t+1},w_t) + \inn{w-w_t,\nabla\Psi(w_t)-\nabla\Psi(w_{t+1})}+\inn{w_t-w_{t+1},\nabla\Psi(w_t)-\nabla\Psi(w_{t+1})}   \\
       & =-D_\Psi(w_{t+1},w_t) + \eta_t\inn{w-w_t, \nabla_w [f(w_t, z_t)]}+\inn{w_t-w_{t+1},\nabla\Psi(w_t)-\nabla\Psi(w_{t+1})}\\
       & = D_\Psi(w_t, w_{t+1})+ \eta_t\inn{w-w_t, \nabla_w [f(w_t, z_t)]}.
\end{align*}
Taking expectations $\ebb_{z_t}$ on both sides, setting $w=w^*$ and noting that $w_t$ is independent of $z_t$, we see the stated identity \eqref{key-identity}. The proof is complete.
\end{proof}

The necessity of the convergence will be derived by using the strong smoothness of $F$ and the strong convexity of $\Psi$ to bound
$\inn{w_t - w^*,\nabla F(w_t)} = \inn{w_t - w^*, \nabla F(w_t) - \nabla F(w^*)}$ by $O(1) D_\Psi(w^*,w_t)$, from which we can apply the identity (\ref{key-identity})
to get necessary conditions by the following inequality
$$
  \ebb_{z_1, \ldots, z_{t}} [D_\Psi(w^*,w_{t+1})] \geq (1- O(\eta_t))\ebb_{z_1, \ldots, z_{t-1}} [D_\Psi(w^*,w_t)]
  +\ebb_{z_1, \ldots, z_{t}}\big[D_\Psi(w_t,w_{t+1})\big].
$$

The sufficiency will be derived by using the strong smoothness of $f$ and the duality $D_\Psi(w_t,w_{t+1})=D_{\Psi^*}(\nabla\Psi(w_{t+1}),\nabla\Psi(w_t))$
to bound $\ebb_{z_t}\big[D_\Psi(w_t,w_{t+1})\big]$ in terms of $\inn{w^*-w_t, \nabla F(w^*)- \nabla F(w_t)}$ and $\ebb_{z_t}[\|\nabla f(w^*, z_t)\|_*^2]$,
from which we can apply the identity (\ref{key-identity}) again to get
\begin{multline*}
  \ebb_{z_1, \ldots, z_{t}}[D_\Psi(w^*,w_{t+1})] \leq \ebb_{z_1, \ldots, z_{t-1}}[D_\Psi(w^*,w_t)] \\
  -\frac{\eta_t}{2} \ebb_{z_1, \ldots, z_{t}}[\inn{w^*-w_t, \nabla F(w^*)- \nabla F(w_t)}]+ O(\eta_t^2)
\end{multline*}
and then use (\ref{convexcontrol}) for bounding $-\inn{w^*-w_t, \nabla F(w^*)-\nabla F(w_t)}$ by $- \Omega \left(D_\Psi(w^*, w_t)]\right)$ to obtain
$$
\ebb_{z_1, \ldots, z_{t}} [D_\Psi(w^*,w_{t+1})] \leq \ebb_{z_1, \ldots, z_{t-1}} [D_\Psi(w^*,w_t)]
- \frac{\eta_t}{2} \Omega \left(\ebb_{z_1, \ldots, z_{t-1}} [D_\Psi(w^*, w_t)]\right) +
O(\eta_t^2).
$$
Here for a continuous convex function $g:\rbb^d\to\rbb$, the Fenchel-conjugate $g^*$ is defined by
$$g^*(v)=\sup_{w\in\wcal}[\inn{w,v}-g(w)], \qquad v\in \rbb^d $$
and the duality (\ref{dualityBregman}) on the Bregman distances is stated (see, e.g., \citep{borwein2010convex}) in the following lemma together with the duality between strong convexity and strong smoothness~\citep{kakade2012regularization}.

\begin{lemma}\label{lem:convexity-smoothness-duality}
Let $g:\rbb^d\to\rbb$ be continuous and convex. Let $\beta >0$. Then $g$ is $\beta$-strongly convex with respect to the norm $\|\cdot\|$ if and only if $g^*$ is $\frac{1}{\beta}$-strongly smooth with respect to the dual norm $\|\cdot\|_*$.

If $g$ is Fr\'{e}chet differentiable and strongly convex, then there holds
\begin{equation}\label{dualityBregman}
  D_g(w,\tilde{w})=D_{g^*}(\nabla g(\tilde{w}),\nabla g(w)),\qquad \forall w,\tilde{w}\in\wcal.
\end{equation}
\end{lemma}

\section{Convergence in the Case of Positive Variances}\label{sec:generalconds}

In this section we prove Theorem \ref{thm:nece-suff} by deriving the necessary and sufficient condition from the following two propositions.

\subsection{Necessary condition for convergence}

The first proposition gives the necessity for the convergence of the OMD algorithm \eqref{online-mirror-descent}.

\begin{proposition}\label{lem:necessary}
Assume $\inf_{w\in\wcal} \ebb_Z \left[\|\nabla_w [f(w, Z)]\|_*\right]>0$ and that $F$ is strongly smooth.
Assume also that $\nabla \Psi$ satisfies the incremental condition (\ref{IncrePsi}) at infinity.
If $\lim_{t\to\infty} \ebb_{z_1, \ldots, z_{t-1}} [D_\Psi(w^*, w_t)]=0$ for some $w^*$ where $\nabla \Psi$ is continuous, then the step size sequence satisfies \eqref{nece-suff}.

Furthermore, if $\Psi$ is strongly smooth, then \eqref{lower-rate} holds with some constants $t_0\in\nbb$ and $\tilde{C}>0$.
\end{proposition}

\begin{proof}
We first show $\lim_{t\to\infty}\eta_t=0$.
By the $\sigma_\Psi$-strong convexity of $\Psi$, we have $\|w^* - w_t\|^2 \leq \frac{2}{\sigma_\Psi} D_\Psi(w^*,w_t)$.
So the condition $\lim_{t\to\infty}\ebb_{z_1, \ldots, z_{t-1}} [D_\Psi(w^*, w_t)]=0$ implies $\lim_{t\to\infty} \ebb_{z_1, \ldots, z_{t-1}} [\|w^* - w_t\|^2]=0$.
Then we claim that
\begin{equation}\label{claimconv}
\lim_{t\to\infty} \ebb_{z_1, \ldots, z_{t-1}} [\|\nabla\Psi(w_t)-\nabla\Psi(w^*)\|_*]=0.
\end{equation}

To prove our claim, we use the continuity of $\nabla\Psi$ at $w^*$ and know that for any $\varepsilon >0$,
there exists some $0<\delta\leq 1$ such that $\|\nabla\Psi(w)-\nabla\Psi(w^*)\|_* < \varepsilon$ whenever $\|w-w^*\|<\delta$.

When $\|w-w^*\|\geq\delta$, we apply the incremental condition (\ref{IncrePsi}) and $\|w\| \leq \|w-w^*\| + \|w^*\|$ to find
$$ \|\nabla\Psi(w)-\nabla\Psi(w^*)\|_* \leq C_\Psi (1+\|w\|) + \|\nabla\Psi(w^*)\|_* \leq C_{\Psi, w^*, \delta} \|w-w^*\|, $$
where $C_{\Psi, w^*, \delta}$ is the constant given by
$$  C_{\Psi, w^*, \delta} = C_\Psi + \frac{C_\Psi + C_\Psi \|w^*\| + \|\nabla\Psi(w^*)\|_*}{\delta}. $$
Combining the above two cases, we know that
\begin{align*}
&\ebb_{z_1, \ldots, z_{t-1}} [\|\nabla\Psi(w_t)-\nabla\Psi(w^*)\|_*] \leq  \varepsilon + C_{\Psi, w^*, \delta} \ebb_{z_1, \ldots, z_{t-1}} [\|w_t-w^*\|].
\end{align*}
But $\lim_{t\to\infty} \ebb_{z_1, \ldots, z_{t-1}} [\|w^* - w_t\|^2]=0$ ensures the existence of some $t_{\varepsilon, \delta} \in \nbb$ such that for $t>t_{\varepsilon, \delta}$, there holds
$\ebb_{z_1, \ldots, z_{t-1}} [\|w_t-w^*\|^2] < \frac{\varepsilon^2}{C^2_{\Psi, w^*, \delta}}$ which implies
$\ebb_{z_1, \ldots, z_{t-1}} [\|w_t-w^*\|] < \frac{\varepsilon}{C_{\Psi, w^*, \delta}}$
by the Schwarz inequality. So we have $\ebb_{z_1, \ldots, z_{t-1}} [\|\nabla\Psi(w_t)-\nabla\Psi(w^*)\|_*] < 2 \varepsilon$ for $t>t_{\varepsilon, \delta}$, which verifies our claim (\ref{claimconv}).

Denote $\sigma = \inf_{w\in\wcal} \ebb_Z \left[\|\nabla_w [f(w, Z)]\|_*\right]>0$. From the iteration relation (\ref{online-mirror-descent}) of the OMD algorithm, we have
$\eta_t \|\nabla_w [f(w_t, z_t)]\|_* =  \|\nabla\Psi(w_t)- \nabla\Psi(w_{t+1})\|_*$. Taking expectations on both sides with respect to $z_t$ yields
$$\eta_t \sigma \leq  \eta_t \ebb_{z_t} \left[\|\nabla_w [f(w_t, z_t)]\|_*\right] \leq \|\nabla\Psi(w_t)- \nabla\Psi(w^*)\|_* + \ebb_{z_t} \left[\|\nabla\Psi(w_{t+1})- \nabla\Psi(w^*)\|_*\right] $$
and
$$\eta_t \sigma \leq  \ebb_{z_1, \ldots, z_{t-1}} [\|\nabla\Psi(w_t)-\nabla\Psi(w^*)\|_*] + \ebb_{z_1, \ldots, z_{t}} [\|\nabla\Psi(w_{t+1})-\nabla\Psi(w^*)\|_*]. $$
Hence (\ref{claimconv}) confirms our first limit $\lim_{t\to\infty}\eta_t=0$.

We now show $\sum_{t=1}^{\infty}\eta_t=\infty$. Assume that $F$ is $L_F$-strongly smooth for some $L_F>0$.
From the identity (\ref{Bregmansum}) and the optimality condition $\nabla F(w^*)=0$, we have
\begin{align*}
D_F(w^*,w_t)+D_F(w_t,w^*) = - \inn{w^*-w_t,\nabla F(w_t)}.
\end{align*}
This is bounded by $L_F\|w^*-w_t\|^2$ by the $L_F$-strong smoothness of $F$. But the $\sigma_\Psi$-strong convexity of $\Psi$ implies $D_\Psi(w^*,w_t) \geq \frac{\sigma_\Psi}{2} \|w^*-w_t\|^2$. Hence
$$
  \inn{w^*-w_t,\nabla F(w_t)}\geq -L_F\|w^*-w_t\|^2\geq -\frac{2L_F}{\sigma_\Psi}D_\Psi(w^*,w_t).
$$
Plugging this inequality into \eqref{key-identity} and taking expectations on both sides give
\begin{equation}\label{necessary-3}
\ebb_{z_1, \ldots, z_{t}} [D_\Psi(w^*,w_{t+1})] \geq (1-a\eta_t)\ebb_{z_1, \ldots, z_{t-1}}[D_\Psi(w^*,w_t)] +\ebb_{z_1, \ldots, z_{t}}[D_\Psi(w_t,w_{t+1})],
\end{equation}
where $a$ is the constant $a=2L_F\sigma_\Psi^{-1}$.

Since $\lim_{t\to\infty}\eta_t=0$, we can find some integer $t_0\in\nbb$ such that $\eta_t\leq(3a)^{-1}$ for $t\geq t_0$. Applying the elementary inequality $1-\eta\geq\exp(-2\eta)$ valid for $\eta\in(0,1/3]$, we know by noting $\ebb_{z_1, \ldots, z_{t}}[D_\Psi(w_t,w_{t+1})] \geq 0$ in (\ref{necessary-3}) that
\begin{equation}\label{necessary-31/2}
\ebb_{z_1, \ldots, z_{t}} [D_\Psi(w^*,w_{t+1})] \geq  \exp(-2a\eta_t) \ebb_{z_1, \ldots, z_{t-1}}[D_\Psi(w^*,w_t)], \qquad \forall t\geq t_0.
\end{equation}
Applying this inequality iteratively for $t=T,\ldots,t_0+1$ then yields
\begin{align}
\ebb_{z_1, \ldots, z_{T}}[D_\Psi(w^*,w_{T+1})] &\geq \prod_{t=t_0+1}^{T}\exp(-2a\eta_t)\ebb_{z_1, \ldots, z_{t_0}}[D_\Psi(w^*,w_{t_0+1})] \notag \\
&=\exp\Big(-2a\sum_{t=t_0+1}^{T}\eta_t\Big)\ebb_{z_1, \ldots, z_{t_0}}[D_\Psi(w^*,w_{t_0+1})]. \label{necessary-4}
\end{align}
We claim that $\ebb_{z_1, \ldots, z_{t_0}}[D_\Psi(w^*,w_{t_0+1})]>0$. Otherwise, we would have
$$\ebb_{z_1, \ldots, z_{t_0 -1}}[D_\Psi(w^*,w_{t_0})] =\ebb_{z_1, \ldots, z_{t_0}}[D_\Psi(w^*,w_{t_0+1})]=0$$ by (\ref{necessary-31/2}), leading to $\ebb_{z_1, \ldots, z_{t_0 -1}}[\|w^* -w_{t_0}\|^2] =\ebb_{z_1, \ldots, z_{t_0}}[\|w^* -w_{t_0+1}\|^2]=0$ according to the strong convexity of $\Psi$. This would imply $w_{t_0+1}=w_{t_0} =w^*$ almost surely and thereby
$\nabla_w [f(w^*, z_{t_0})] =0$ almost surely by (\ref{online-mirror-descent}), leading to $\ebb_Z \left[\|\nabla_w [f(w^*, Z)]\|_*\right]=0$, a contradiction to the assumption $\inf_{w\in\wcal} \ebb_Z \left[\|\nabla_w [f(w, Z)]\|_*\right]>0$.

By $\ebb_{z_1, \ldots, z_{t_0}}[D_\Psi(w^*,w_{t_0+1})]>0$ and
the limit $\lim_{T\to\infty}\ebb_{z_1, \ldots, z_{T}}[D_\Psi(w^*,w_{T+1})]=0$, we see from \eqref{necessary-4} that $\sum_{t=1}^{\infty}\eta_t=\infty$. This proves the necessary condition for the convergence of the OMD algorithm.

We now prove \eqref{lower-rate} under the $L_\Psi$-strong smoothness of $\Psi$ for some $L_\Psi>0$.
Since $\Psi$ is $\sigma_\Psi$-strongly convex and $L_\Psi$-strongly smooth with respect to $\|\cdot\|$,
we know from Lemma \ref{lem:convexity-smoothness-duality} that $\Psi^*$ is $\sigma_\Psi^{-1}$-strongly smooth
and $L_\Psi^{-1}$-strongly convex with respect to $\|\cdot\|_*$ (note $\Psi^{**}=\Psi$ since $\Psi$ is convex).
We also know from Lemma \ref{lem:convexity-smoothness-duality} that the duality relation (\ref{dualityBregman}) between Bregman distances holds for $g=\Psi$, which yields
$$
  D_\Psi(w_t,w_{t+1})=D_{\Psi^*}(\nabla\Psi(w_{t+1}),\nabla\Psi(w_t)),\quad\forall t\in\nbb.
$$
Combining this with the $L_\Psi^{-1}$-strong convexity of $\Psi^*$ and (\ref{necessary-3}), we know from the bound $\eta_t\leq (3a)^{-1}$ that for $t\geq t_0$,
\begin{multline*}
\ebb_{z_1, \ldots, z_{t}} [D_\Psi(w^*,w_{t+1})] \geq (1-a\eta_t)\ebb_{z_1, \ldots, z_{t-1}} [D_\Psi(w^*,w_t)] \\ + (2L_\Psi)^{-1}\ebb_{z_1, \ldots, z_{t}} \big[\|\nabla\Psi(w_t)-\nabla\Psi(w_{t+1})\|_*^2\big].
\end{multline*}
But $\nabla\Psi(w_t)-\nabla\Psi(w_{t+1}) = \eta_t\nabla_w [f(w_t, z_t)]$ by the definition (\ref{online-mirror-descent}) of the OMD algorithm. So for $t\geq t_0$ we have
\begin{multline*}
\ebb_{z_1, \ldots, z_{t}} [D_\Psi(w^*,w_{t+1})] \geq (1-a\eta_t)\ebb_{z_1, \ldots, z_{t-1}} [D_\Psi(w^*,w_t)] \\
+ (2L_\Psi)^{-1} \eta_t^2 \ebb_{z_1, \ldots, z_{t}} \big[\|\nabla_w [f(w_t, z_t)]\|_*^2\big].
\end{multline*}
By the Schwarz inequality,
$$\ebb_{z_1, \ldots, z_{t}} \big[\|\nabla_w [f(w_t, z_t)]\|_*\big] \leq \left\{\ebb_{z_1, \ldots, z_{t}} \big[\|\nabla_w [f(w_t, z_t)]\|_*^2\big]\right\}^{1/2}. $$
Hence
$$\ebb_{z_1, \ldots, z_{t}} \big[\|\nabla_w [f(w_t, z_t)]\|_*^2\big] \geq \left\{\ebb_{z_1, \ldots, z_{t}} \big[\|\nabla_w [f(w_t, z_t)]\|_*\big]\right\}^2
\geq \sigma^2 $$
and thereby
$$
\ebb_{z_1, \ldots, z_{t}} [D_\Psi(w^*,w_{t+1})] \geq (1-a\eta_t)\ebb_{z_1, \ldots, z_{t-1}} [D_\Psi(w^*,w_t)] + (2L_\Psi)^{-1} \eta_t^2 \sigma^2,\quad\forall t\geq t_0.
$$
Applying this inequality iteratively from $t=T\geq t_0$ to $t=t_0$ yields (denote $\prod_{k=T+1}^{T}(1-a\eta_k)=1$)
\begin{align*}
&\ebb_{z_1, \ldots, z_{T}} [D_\Psi(w^*,w_{T+1})]\\
& \geq \ebb_{z_1, \ldots, z_{t_0 -1}} [D_\Psi(w^*,w_{t_0})] \prod_{t=t_0}^T (1- a\eta_t) + (2L_\Psi)^{-1} \sigma^2 \sum_{t=t_0}^{T} \eta_t^2 \prod_{k=t+1}^{T}(1-a \eta_k) \\
   & \geq (2L_\Psi)^{-1} \sigma^2 \sum_{t=t_0}^{T} \eta_t^2 \prod_{k=t+1}^{T}(1-a \eta_k).
\end{align*}
By the Schwarz inequality and the bound $0 < 1-a \eta_k \leq 1$ for $k\geq t_0$, we have
$$ \sum_{t=t_0}^{T} \eta_t \prod_{k=t+1}^{T}(1-a \eta_k)
\leq \left\{\sum_{t=t_0}^{T} \eta_t^2 \prod_{k=t+1}^{T}(1-a \eta_k)\right\}^{1/2} (T-t_0 +1)^{1/2}. $$
Hence
\begin{align*}
\sum_{t=t_0}^{T} \eta_t^2 \prod_{k=t+1}^{T}(1-a \eta_k) &\geq \frac{1}{a^2(T-t_0 +1)} \left(\sum_{t=t_0}^{T} a \eta_t \prod_{k=t+1}^{T}(1-a \eta_k)\right)^2 \\
&= \frac{1}{a^2(T-t_0 +1)} \left(\sum_{t=t_0}^{T} \bigl(1- (1-a \eta_t)\bigr) \prod_{k=t+1}^{T}(1-a \eta_k)\right)^2 \\
&= \frac{1}{a^2(T-t_0 +1)} \left(\sum_{t=t_0}^{T} \left[\prod_{k=t+1}^{T}(1-a \eta_k) - \prod_{k=t}^{T}(1-a \eta_k)\right]\right)^2 \\
&= \frac{1}{a^2(T-t_0 +1)} \left(1 - \prod_{k=t_0}^{T}(1-a \eta_k)\right)^2 \\
&\geq \frac{1}{a^2(T-t_0 +1)} \left(1 - (1-a \eta_{t_0})\right)^2 = \frac{\eta_{t_0}^2}{T-t_0 +1}.
\end{align*}
Therefore,
$$ \ebb_{z_1, \ldots, z_{T}} [D_\Psi(w^*,w_{T+1})] \geq \frac{\eta_{t_0}^2 (2L_\Psi)^{-1} \sigma^2}{T-t_0 +1}, \qquad \forall T \geq t_0. $$
This verifies \eqref{lower-rate} with $\tilde{C}=\eta_{t_0}^2 (2L_\Psi)^{-1} \sigma^2$ and completes the proof.
\end{proof}

\subsection{Sufficient condition for convergence}

We now turn to the second proposition giving the sufficiency for the convergence of the OMD \eqref{online-mirror-descent}.
We need the following lemma, to be proved in appendix by some ideas from \citep{ying2015unregularized}, which establishes the co-coercivity of gradients for convex functions enjoying some smoothness condition.

\begin{lemma}\label{lem:co-coercivity}
   Let $\alpha\in(0,1]$ and $g:\wcal\to\rbb$ be a Fr\'{e}chet differentiable and convex function. If there exists some constant $L>0$ such that
$$
    D_g(w,\tilde{w})\leq \frac{L}{1+\alpha}\|w-\tilde{w}\|^{1+\alpha},\quad\forall w,\tilde{w}\in\wcal,
$$
then we have
  \begin{equation}\label{sufficient-condition-b}
\frac{2L^{-\frac{1}{\alpha}}\alpha}{1+\alpha}\|\nabla g(w)-\nabla g(\tilde{w})\|_*^{\frac{1+\alpha}{\alpha}} \leq  \inn{w-\tilde{w},\nabla g(w)-\nabla g(\tilde{w})},\qquad\forall w,\tilde{w}\in\wcal.
  \end{equation}
\end{lemma}

\begin{proposition}\label{lem:sufficient}
Assume that for some constant $L>0$, $f(\cdot, z)$ is $L$-strongly smooth for almost every $z\in Z$.
Suppose that the pair $(\Psi, F)$ satisfies (\ref{convexcontrol}) around $w^*$ with a convex function $\Omega: [0, \infty) \to \rbb_+$ satisfying $\Omega (0)=0$ and $\Omega (u)>0$ for $u>0$. If the step size sequence satisfies \eqref{nece-suff}, then $\lim_{t\to\infty}\ebb_{z_1, \ldots, z_{t-1}} [D_\Psi(w^*,w_t)]=0$.

Furthermore, if \eqref{strong-convexity-assumption} holds with some $\sigma_F>0$ and
    the step size takes the form $\eta_t=\frac{4}{(t+1)\sigma_F}$,
    then \eqref{one-over-T} holds.
\end{proposition}
\begin{proof}
According to the key identity \eqref{key-identity} for the one-step progress of the OMD algorithm
and the duality relation \eqref{dualityBregman} of the Bregman distances, we have
\begin{equation}\label{sufficient-extra-first}
  \ebb_{z_t} [D_\Psi(w^*,w_{t+1})]-D_\Psi(w^*,w_t) = \eta_t\inn{w^*-w_t, \nabla F(w_t)} +
  \ebb_{z_t} \big[D_{\Psi^*}(\nabla\Psi(w_{t+1}),\nabla\Psi(w_t))\big].
\end{equation}
By Lemma \ref{lem:convexity-smoothness-duality}, the $\sigma_\Psi$-strong convexity of $\Psi$ implies the $\sigma_\Psi^{-1}$-strong smoothness of $\Psi^*$. It follows from the definition (\ref{online-mirror-descent}) of the OMD algorithm that
\begin{align}
\ebb_{z_t} \big[D_{\Psi^*}(\nabla\Psi(w_{t+1}),\nabla\Psi(w_t))\big] &\leq \frac{1}{2\sigma_\Psi} \ebb_{z_t}\big[\|\nabla\Psi(w_{t+1})-\nabla\Psi(w_t)\|_*^2\big] \notag\\
&= \frac{\eta_t^2}{2\sigma_\Psi} \ebb_{z_t}\big[\|\nabla_w [f(w_t, z_t)]\|_*^2\big].\label{sufficient-extra}
\end{align}
We bound $\big[\|\nabla_w [f(w_t, z_t)]\|_*^2\big]$ by $2\big[\|\nabla_w [f(w_t, z_t)] - \nabla_w [f(w^*, z_t)]\|_*^2\big]
+ 2 \big[\|\nabla_w [f(w^*, z_t)]\|_*^2\big]$. Then we apply Lemma \ref{lem:co-coercivity} with $w=w^*, \tilde{w}=w_t, g=f(\cdot, z_t)$ and $\alpha=1$.
By the $L$-strong smoothness of $f(\cdot, z)$, we know that
\begin{align*}
  \ebb_{z_t}\Big[\|\nabla_w [f(w_t, z_t)] - \nabla_w [f(w^*, z_t)]\|_*^2\Big]
  &\leq L\ebb_{z_t}\Big[\big\langle w_t - w^*, \nabla_w[f(w_t,z_t)] - \nabla_w[f(w^*,z_t)]\big\rangle\Big] \\
  & = L \inn{w^*-w_t, \nabla F(w^*) - \nabla F(w_t)}.
\end{align*}
Then we have
\begin{multline*}
\ebb_{z_t} [D_\Psi(w^*,w_{t+1})]-D_\Psi(w^*,w_t)
\leq \\
- \left(1- \frac{L\eta_t}{\sigma_\Psi}\right) \eta_t\inn{w^*-w_t, \nabla F(w^*) - \nabla F(w_t)}  +
\frac{\eta_t^2}{\sigma_\Psi} \ebb_{z_t} \big[\|\nabla_w [f(w^*, z_t)]\|_*^2\big].
\end{multline*}

Since $\lim_{t\to\infty} \eta_t =0$, there exists some $t_1 \in\nbb$ such that $\frac{L}{\sigma_\Psi}\eta_t \leq \frac{1}{2}$ for $t\geq t_1$ which implies
\begin{multline}\label{localestimate}
 \ebb_{z_t} [D_\Psi(w^*,w_{t+1})]-D_\Psi(w^*,w_t)
\leq \\
- \frac{\eta_t}{2} \inn{w^*-w_t, \nabla F(w^*) - \nabla F(w_t)} + \frac{\eta_t^2}{\sigma_\Psi} \ebb_{z_t} \big[\|\nabla_w [f(w^*, z_t)]\|_*^2\big].
\end{multline}
Now we apply the relation (\ref{convexcontrol}) on the convexity to obtain
\begin{equation}\label{convexcontrol-used}
 -\inn{w^*-w_t, \nabla F(w^*)-\nabla F(w_t)} \leq - \Omega \left(D_\Psi(w^*, w_t)\right).
\end{equation}
It follows that
$$
\ebb_{z_t} [D_\Psi(w^*,w_{t+1})] \leq D_\Psi(w^*,w_t)
- \frac{\eta_t}{2} \Omega \left(D_\Psi(w^*, w_t)\right) +
b \eta_t^2,
$$
where $b$ is the constant $b = \frac{1}{\sigma_\Psi} \ebb_{Z} \big[\|\nabla_w [f(w^*, Z)]\|_*^2\big]$.
Since $\Omega$ is convex, by Jensen's inequality, we have
$$\Omega \left(\ebb_{z_1, \ldots, z_{t-1}} [D_\Psi(w^*, w_t)]\right) \leq
\ebb_{z_1, \ldots, z_{t-1}} \left[\Omega \left(D_\Psi(w^*, w_t)\right)\right].
$$
Therefore, by taking expectations over $z_1, \ldots, z_{t-1}$ and denoting a sequence $\{A_t\}_t$ by
$$A_t = \ebb_{z_1, \ldots, z_{t-1}} \left[D_\Psi(w^*, w_t)\right], $$
we have
\begin{equation}\label{sufficient-4}
A_{t+1} \leq A_t
- \frac{\eta_t}{2} \Omega \left(A_t\right) +
b \eta_t^2, \qquad \forall t\geq t_1.
\end{equation}

To prove $\lim_{t\to\infty} A_t =0$, we let $0<\gamma <1$ be an arbitrarily chosen number.
The convexity of $\Omega: [0, \infty) \to \rbb_+$ tells us that for $u\geq \gamma$, there holds
$$ \Omega (\gamma) = \Omega\left((1-\frac{\gamma}{u}) \cdot 0 + \frac{\gamma}{u} u\right) \leq
(1-\frac{\gamma}{u}) \Omega\left(0\right) + \frac{\gamma}{u} \Omega (u) = \frac{\gamma}{u} \Omega (u) $$
which yields
\begin{equation}\label{sufficient-newmore}
\Omega (u) \geq \frac{\Omega (\gamma)}{\gamma} u, \qquad \forall u \geq \gamma.
\end{equation}
Since $\lim_{t\to\infty}\eta_t=0$, we know that there exists some integer $t_\gamma \geq t_1$ such that
\begin{equation}\label{sufficient-3}
  \eta_t \leq \min\left\{\frac{\Omega (\gamma)}{4 b}, \frac{\Omega (\gamma)}{4 \gamma b}, \sqrt{\gamma}\right\}, \qquad\forall t\geq t_\gamma.
\end{equation}

We claim that
\begin{equation}\label{claiminfty}
\sup\left\{t \in \nbb : A_t \leq \gamma\right\} =\infty.
\end{equation}
If (\ref{claiminfty}) is not true, we can find some $t'_\gamma \geq t_\gamma$ such that
$$ A_t > \gamma, \qquad \forall t \geq t'_\gamma. $$
Combining this with (\ref{sufficient-newmore}), (\ref{sufficient-3}) and (\ref{sufficient-4}) tells us that for $t\geq t'_\gamma$,
$$ A_{t+1} \leq A_t - \eta_t \frac{\Omega (\gamma)}{2 \gamma} A_t  +  b \eta_t^2
\leq A_t -\frac{\Omega (\gamma)}{2 \gamma} \eta_t A_t  + \frac{\Omega (\gamma)}{4 \gamma} \eta_t  A_t
= A_t -\frac{\Omega (\gamma)}{4\gamma} \eta_t A_t \leq A_t
-\frac{\Omega (\gamma)}{4} \eta_t, $$
which implies by iteration
$$ A_{t+1} \leq A_{t'_\gamma} -  \frac{\Omega (\gamma)}{4} \sum_{k=t'_\gamma}^t \eta_k \to -\infty \; (\hbox{as} \ t \to \infty). $$
This is a contradiction, which verifies our claim (\ref{claiminfty}).

By (\ref{claiminfty}) there exists some positive integer $t''_\gamma > t_\gamma$ such that
$A_{t''_\gamma} \leq \gamma$. We now show by induction that
\begin{equation}\label{sufficient-7}
A_t \leq  \gamma + b \max_{t''_\gamma \leq \ell \leq t-1}\eta_\ell^2, \qquad\forall t\geq t''_\gamma.
\end{equation}
The case $t=t''_\gamma$ is true (where we denote $\max_{t''_\gamma \leq \ell \leq t''_\gamma-1}\eta_\ell^2 =0$) since $A_{t''_\gamma} \leq \gamma$.
Supposes the statement \eqref{sufficient-7} holds for $t=k \geq t''_\gamma$. Note that $t''_\gamma > t_\gamma$ and $\gamma <1$.
To prove the statement for $t=k+1$, we discuss in two cases.
If $A_k \leq\gamma$, we see directly from (\ref{sufficient-4}) that
$$ A_{k+1} \leq \gamma +  b \eta_k^2 \leq \gamma + b \max_{t''_\gamma \leq \ell \leq k}\eta_\ell^2.
$$
If $A_k > \gamma$, we apply (\ref{sufficient-newmore}), (\ref{sufficient-3}) and (\ref{sufficient-4}) again and find
$$ A_{k+1} \leq A_k - \eta_k \frac{\Omega (\gamma)}{2 \gamma} A_k  +  b \eta_k^2
\leq A_k -\frac{\Omega (\gamma)}{4\gamma} \eta_k A_k
\leq A_k \leq \gamma + b \max_{t''_\gamma \leq \ell \leq k-1} \eta_\ell^2, $$
where we have used the induction hypothesis in the last inequality. This verifies the statement \eqref{sufficient-7} for $t=k +1$ and completes the induction procedure.

Applying (\ref{sufficient-3}), \eqref{sufficient-7} and noting $t''_\gamma > t_\gamma$, we know that
$$ A_t \leq  (1 + b) \gamma,\qquad\forall t\geq t''_\gamma. $$
Since $\gamma$ is an arbitrary number on $(0, 1)$, this proves
$$\lim_{t\to\infty} A_t =\lim_{t\to\infty} \ebb_{z_1, \ldots, z_{t-1}} \left[D_\Psi(w^*, w_t)\right]=0.$$

We now prove \eqref{one-over-T} under condition \eqref{strong-convexity-assumption} and the choice
$\eta_t=\frac{4}{(t+1)\sigma_F}$ of the step size sequence. Here $\Omega (u) = \sigma_F u$ and the estimate \eqref{sufficient-4} becomes
$$
A_{t+1} \leq A_t - \frac{2}{t+1} A_t +  \frac{16 b}{(t+1)^2\sigma_F^2}, \qquad \forall t \geq t_1.
$$
It follows that
$$
t (t+1) A_{t+1} \leq (t-1)t A_t + \frac{16 b}{\sigma_F^2}, \qquad\forall t\geq t_1.
$$
Applying this relation iteratively, we obtain
$$
(T-1) T A_{T} \leq (t_1-1)t_1 A_{t_1} + \frac{16 b (T- t_1)}{\sigma_F^2}, \qquad\forall T\geq t_1,
$$
from which we see
$$
\ebb_{z_1, \ldots, z_{T-1}} [D_\Psi(w^*, w_{T})] \leq \frac{(t_1-1)t_1 \ebb_{z_1, \ldots, z_{t_1 -1}} [D_\Psi(w^*, w_{t_1})]}{(T-1) T} + \frac{16b}{T\sigma_F^2},\qquad \forall T\geq t_1.
$$
This yields \eqref{one-over-T}. The proof is complete.
\end{proof}

\section{Convergence in the Case of Zero Variances and Almost Sure Convergence}\label{sec:zerovar}

In this section we prove Theorem \ref{thm:nece-suff-without-variance} for the convergence in the case of zero variances and Theorem \ref{thm:ae-sufficient} for the almost sure convergence.

\begin{proof}[Proof of Theorem \ref{thm:nece-suff-without-variance}]
Necessity. The assumption that $f(\cdot, z)$ is $L$-strongly smooth for almost every $z\in\zcal$ implies the $L$-strong smoothness of $F$.
We observe that the estimate (\ref{necessary-3}) derived in the proof of Proposition \ref{lem:necessary} is valid under the $L_F$-strong smoothness of $F$ and the $\sigma_\Psi$-strong convexity of $\Psi$. Hence
\begin{equation}\label{linear-1}
\ebb_{z_1, \ldots, z_{t}} [D_\Psi(w^*,w_{t+1})]  \geq (1-2L\sigma_\Psi^{-1}\eta_t) \ebb_{z_1, \ldots, z_{t-1}}[D_\Psi(w^*,w_t)].
\end{equation}

We now need the assumption $0<\eta_t\leq\frac{\sigma_\Psi}{(2+\kappa)L}$ with $\kappa>0$ on the step size sequence. Denote the constant $\tilde{a}=\frac{2+\kappa}{2}\log\frac{2+\kappa}{\kappa}$
and apply the elementary inequality (see e.g.,~\citep{lei2016analysisb})
$$1-x\geq \exp(-\tilde{a}x), \quad\forall 0< x \leq \frac{2}{2+\kappa}. $$
We know from \eqref{linear-1} that
$$
\ebb_{z_1, \ldots, z_{t}} [D_\Psi(w^*,w_{t+1})] \geq \exp\big(-2\tilde{a}L\sigma_\Psi^{-1}\eta_t\big)\ebb_{z_1, \ldots, z_{t-1}}[D_\Psi(w^*,w_t)].
$$
Applying this inequality iteratively for $t=1,\ldots,T$ then gives
\begin{align*}
\ebb_{z_1, \ldots, z_{T}} [D_\Psi(w^*,w_{T+1})] &\geq \prod_{t=1}^{T}\exp\big(-2\tilde{a}L\sigma_\Psi^{-1}\eta_t\big) D_\Psi(w^*, w_1) \\
&= \exp\left\{-2\tilde{a}L\sigma_\Psi^{-1} \sum_{t=1}^{T}\eta_t\right\} D_\Psi(w^*, w_1).
\end{align*}
From the assumption $w^*\neq w_1$, we have $D_\Psi(w^*,w_1)>0$. The convergence $\lim_{t\to\infty}\ebb_{z_1, \ldots, z_{t-1}} [D_\Psi(w^*,w_{t})]  =0$ then implies $\sum_{t=1}^{\infty}\eta_t=\infty$.

Sufficiency. Here we use the estimates (\ref{sufficient-4}) derived in the proof of Proposition \ref{lem:sufficient}.
But in our case of zero variances, $b = \frac{1}{\sigma_\Psi} \ebb_{Z} \big[\|\nabla_w [f(w^*, Z)]\|_*^2\big] =0$.
So (\ref{sufficient-4}) takes the form (note that we can choose $t_1=1$ in deriving \eqref{localestimate})
\begin{equation}\label{sufficient-4new}
A_{t+1} \leq A_t
- \frac{\eta_t}{2} \Omega \left(A_t\right), \qquad \forall t\in\nbb.
\end{equation}
This implies that for any $0<\gamma <1$, there must exist some integer $\tilde{t}_\gamma\in\nbb$ such that $A_{\tilde{t}_\gamma} \leq \gamma$, since otherwise $A_t > \gamma$ for every $t \in\nbb$, which by (\ref{sufficient-newmore}) and (\ref{sufficient-4new}) leads to a contradiction:
$$ A_{t+1} \leq A_t -\frac{\eta_t \Omega (\gamma)}{2 \gamma} A_t
\leq A_t -\frac{\eta_t}{2} \Omega (\gamma) \leq A_{\tilde{t}_\gamma} -  \frac{\Omega (\gamma)}{2} \sum_{k=\tilde{t}_\gamma}^t \eta_k \to -\infty \ (\hbox{as} \ t \to \infty). $$
But (\ref{sufficient-4new}) also tells us that the sequence $\{A_t\}_{t \in\nbb}$ of nonnegative numbers is decreasing. Hence $A_{\tilde{t}_\gamma} \leq \gamma$ for every $t \geq \tilde{t}_\gamma$. This proves the limit
$$\lim_{t\to\infty} \ebb_{z_1, \ldots, z_{t-1}} \left[D_\Psi(w^*, w_t)\right] = \lim_{t\to\infty} A_t =0.$$

We now turn to prove \eqref{linear-rate} under the special choice of the constant step size sequence $\eta_t \equiv \eta_1$. It follows from \eqref{linear-1} that
$A_{T+1}  \geq (1-2L\sigma_\Psi^{-1}\eta_1)^TA_1$.
Furthermore, under the assumption \eqref{strong-convexity-assumption}, we have $\Omega (u) = \sigma_F u$. So \eqref{sufficient-4new} translates to
$$A_{t+1} \leq (1-2^{-1}\eta_1\sigma_F)A_t, $$
from which we find $A_{T+1}\leq (1-2^{-1}\eta_1 \sigma_F)^TA_1$ by iteration. This verifies \eqref{linear-rate} and completes the proof of Theorem \ref{thm:nece-suff-without-variance}.
\end{proof}

The proof of Theorem \ref{thm:ae-sufficient} for the almost sure convergence is based on the following Doob's forward convergence theorem (see, e.g.,~\citep{Doob1994} on page 195).

\begin{lemma}\label{lem:super-martingale}
  Let $\{\tilde{X}_t\}_{t\in\nbb}$ be sequences of nonnegative random variables and let $\{\fcal_t\}_{t\in\nbb}$ be a sequence of random variable sets with $\fcal_t\subset\fcal_{t+1}$ for every $t\in\nbb$. Suppose that
 $\ebb[\tilde{X}_{t+1}|\fcal_t]\leq \tilde{X}_t$ almost surely for every $t\in\nbb$. Then the sequence $\{\tilde{X}_t\}$ converges to a nonnegative random variable $\tilde{X}$ almost surely.
\end{lemma}

\begin{proof}[Proof of Theorem \ref{thm:ae-sufficient}]
We follow the proof of Proposition \ref{lem:sufficient} and apply (\ref{localestimate}). Since $\inn{w^*-w_t, \nabla F(w^*) - \nabla F(w_t)}\geq 0$, (\ref{localestimate}) implies
\begin{equation}\label{localestimateas}
 \ebb_{z_t} [D_\Psi(w^*,w_{t+1})] \leq D_\Psi(w^*,w_t) + \frac{\eta_t^2}{\sigma_\Psi} \ebb_{Z} \big[\|\nabla_w [f(w^*, Z)]\|_*^2\big], \qquad \forall t \geq t_1.
\end{equation}
The condition $\sum_{t=1}^{\infty}\eta_t^2<\infty$ enables us to define a stochastic process $\{\tilde{X}_t\}_t$ by
$$ \tilde{X}_t = D_\Psi(w^*,w_{t+1}) + \frac{1}{\sigma_\Psi} \ebb_{Z} \big[\|\nabla_w [f(w^*, Z)]\|_*^2\big] \sum_{\ell = t+1}^\infty \eta_\ell. $$
By (\ref{localestimateas}), we know that $\ebb_{z_t}[\tilde{X}_{t+1}] \leq \tilde{X}_t$ for $t\geq t_1$. Also, $\tilde{X}_t \geq 0$. So the stochastic process $\{\tilde{X}_t\}_{t \geq t_1}$ is a supermartingale.
Then by the supermartingale convergence theorem, Lemma \ref{lem:super-martingale}, we know that the sequence $\{\tilde{X}_t\}_{t \geq t_1}$ converges to a non-negative random variable $\tilde{X}$ almost surely. According to Fatou's Lemma and the limit $\lim_{t\to\infty} \ebb[D_\Psi(w^*,w_t)] =0$ proved by Proposition \ref{lem:sufficient}, we get
  $$
    \ebb[\tilde{X}]=\ebb\big[\lim_{t\to\infty}D_\Psi(w^*,w_t)\big]\leq\liminf_{t\to\infty}\ebb[D_\Psi(w^*,w_t)]=0.
  $$
But $\tilde{X}$ is a non-negative random variable, so we have $\tilde{X}=0$ almost surely. It follows that $\{D_\Psi(w^*,w_t)\}_{t\in\nbb}$ converges to $0$ almost surely. The proof of Theorem \ref{thm:ae-sufficient} is complete.
\end{proof}

\section{Proving Explicit Results}\label{sec:proof-incremental-convex}

In this section we prove the propositions stated in Section \ref{sec:incremental-convex} on some properties of special mirror maps,
and Theorems \ref{thm:least-square} and \ref{thm:p-norm-convergence} on necessary and sufficient conditions for the convergence, as well as tight convergence rates.

\begin{proof}[Proof of Proposition \ref{prop:least-squares}]
If $\Psi$ is $L_\Psi$-strongly smooth, then the condition in Lemma \ref{lem:co-coercivity} is satisfied with $g=\Psi, L=L_\Psi$ and $\alpha=1$. So by Lemma \ref{lem:co-coercivity}, there holds
$$ \|\nabla\Psi(w)-\nabla\Psi(\tilde{w})\|_*^2 \leq L_\Psi\inn{w-\tilde{w},\nabla\Psi(w)-\nabla\Psi(\tilde{w})}, \qquad \forall w, \tilde{w} \in\wcal. $$
By the Schwarz inequality $\inn{w-\tilde{w},\nabla\Psi(w)-\nabla\Psi(\tilde{w})} \leq \|w-\tilde{w}\|\|\nabla\Psi(w)-\nabla\Psi(\tilde{w})\|_*$, this implies
\begin{equation}\label{gradcont}
  \|\nabla\Psi(w)-\nabla\Psi(\tilde{w})\|_*\leq L_\Psi\|w-\tilde{w}\|,\quad\forall w,\tilde{w}\in\wcal.
\end{equation}
So the function $\nabla\Psi$ is Lipschitz, and hence is continuous everywhere.

Setting $\tilde{w}=0$ in (\ref{gradcont}) also yields
$$
  \|\nabla\Psi(w)\|_*\leq \|\nabla\Psi(0)\|_*+L_\Psi\|w\| \leq \left(\|\nabla\Psi(0)\|_*+L_\Psi\right) (1+\|w\|),\qquad\forall w\in\wcal.
$$
This establishes the incremental conditional \eqref{IncrePsi} at infinity with $C_\Psi=\|\nabla\Psi(0)\|_*+L_\Psi$.

If $F$ is $\sigma_F$-strongly convex, by the identity (\ref{Bregmansum}), we have
$$\inn{w -\tilde{w}, \nabla F (w) -\nabla F (\tilde{w}} =  D_F (w,\tilde{w}) +D_F (\tilde{w}, w) \geq \sigma_F \|w -\tilde{w}\|^2, \qquad \forall w, \tilde{w}\in\wcal. $$
But $D_\Psi (\tilde{w}, w) \leq \frac{L_\Psi}{2} \|w -\tilde{w}\|^2$. So we have
$$ \inn{w -\tilde{w}, \nabla F (w) -\nabla F (\tilde{w}} \geq \sigma_F \|w -\tilde{w}\|^2 \geq \frac{2\sigma_F}{L_\Psi}D_\Psi (\tilde{w}, w), \qquad \forall w, \tilde{w}\in\wcal. $$
Hence \eqref{convexcontrol} is satisfied for a linear convex function $\Omega (u) = \frac{2\sigma_F}{L_\Psi} u$. This proves Proposition \ref{prop:least-squares}.
\end{proof}

For proving Proposition \ref{prop:p-divergence}, we need the following inequalities which follow easily from the elementary inequalities
$$
|a^\beta-b^\beta| \leq |a-b|^\beta, \quad (a+b)^\beta \leq a^\beta+b^\beta\leq 2^{1-\beta}(a+b)^\beta, \qquad \forall a, b\geq0, \beta\in (0,1].
$$

\begin{lemma}\label{lem:elementarybeta}
Let $0< \beta \leq 1$. Then we have
\begin{eqnarray}
&&|\sgn(a)|a|^\beta-\sgn(b)|b|^\beta| \leq 2^{1-\beta}|a-b|^\beta, \qquad \forall a, b\in\rbb, \label{p-mirror-4} \\
&& \big|\|\tilde{w}\|_p^{\beta}-\|w\|_p^{\beta}\big| \leq \big|\|\tilde{w}\|_p-\|w\|_p\big|^{\beta} \leq \|\tilde{w}-w\|_p^{\beta}, \qquad \forall w, \tilde{w}\in\wcal, \label{p-mirror-2a}
\end{eqnarray}
where we denote the sign of $a\in\rbb$ by $\sgn(a)=1$ if $a>0$, $-1$ if $a<0$, and $0$ if $a=0$.
\end{lemma}

\begin{proof}[Proof of Proposition \ref{prop:p-divergence}]
Let $p^* =\frac{p}{p-1} >2$ be the dual number of $p$ satisfying $\frac{1}{p} + \frac{1}{p^*} =1$. Then the dual norm $\|\cdot\|_*$ is exactly the $p^*$-norm $\|\cdot\|_{p^*}$, and the gradient of $\Psi_p$ at $w\in\wcal$ equals
\begin{equation}\label{gradient-p-divergence}
    \nabla\Psi_p(w)=\|w\|_p^{2-p} \hat{w},
\end{equation}
where $\hat{w}\in \wcal^*$ is the vector depending on $w$ given by
$$ \hat{w} = \big(\sgn(w(j))|w(j)|^{p-1}\big)_{j=1}^d. $$
It follows that $\nabla\Psi_p$ is continuous everywhere, and by calculating the norm $\big\|\hat{w}\big\|_{p^*}$ directly that
$$
    \|\nabla\Psi_p(w)\|_*=\|w\|_p^{2-p}\big\|\hat{w}\big\|_{p^*}=\|w\|_p^{2-p+\frac{p}{p^*}}=\|w\|_p.
$$
This proves the identity (\ref{nablaPsinorms}) and the incremental condition \eqref{IncrePsi} with $C_{\Psi_p}=1$.

To bound the Bregman distance $D_{\Psi_p}(\tilde{w},w)$, we apply the identity (\ref{Bregmansum}) and find that for any $w, \tilde{w} \in\wcal$,
\begin{equation}\label{intermDPsi}
 D_{\Psi_p}(\tilde{w}, w) \leq D_{\Psi_p}(\tilde{w}, w)+D_{\Psi_p}(w, \tilde{w}) \leq \|\tilde{w}-w\|_p \big\|\nabla\Psi_p(\tilde{w})-\nabla\Psi_p(w)\big\|_{p^*}.
\end{equation}

We use the expression (\ref{gradient-p-divergence}) and write $\nabla\Psi_p(\tilde{w})-\nabla\Psi_p(w)$ as
$$ \nabla\Psi_p(\tilde{w})-\nabla\Psi_p(w) = \|\tilde{w}\|_p^{2-p} \hat{\tilde{w}} - \|w\|_p^{2-p} \hat{w} = \|\tilde{w}\|_p^{2-p} \left(\hat{\tilde{w}} - \hat{w}\right) + \left(\|\tilde{w}\|_p^{2-p} - \|w\|_p^{2-p}\right) \hat{w}. $$
Applying (\ref{p-mirror-4}) to the $j$-th components of $\hat{\tilde{w}} - \hat{w}$ and $\beta =p-1 \in (0, 1)$, we have
$$ \left|\sgn(\tilde{w}(j))|\tilde{w}(j)|^{p-1} - \sgn(w(j))|w(j)|^{p-1}\right| \leq 2^{2-p} \left|\tilde{w}(j) - w(j)\right|^{p-1}, \qquad j=1, \ldots, d. $$
So for the first term, we have
\begin{align}
\left\|\hat{\tilde{w}} - \hat{w}\right\|_{p^*} &\leq \left\{\sum_{j=1}^d 2^{p^*(2-p)} \left|\tilde{w}(j) - w(j)\right|^{p^*(p-1)}\right\}^{\frac{1}{p^*}}\notag\\
 & = 2^{2-p}\left\|\tilde{w} - w\right\|_{p}^{\frac{p}{p^*}} = 2^{2-p} \left\|\tilde{w} - w\right\|_{p}^{p-1}.\label{firsttermnablaPsi}
\end{align}
For the second term, we apply (\ref{p-mirror-2a}) with $\beta =2-p$ and find
$$ \left\|\left(\|\tilde{w}\|_p^{2-p} - \|w\|_p^{2-p}\right) \hat{w}\right\|_{p^*} \leq \|\tilde{w}-w\|_p^{2-p} \left\|\hat{w}\right\|_{p^*} = \|\tilde{w}-w\|_p^{2-p} \left\|w\right\|_p^{p-1}. $$
Applying (\ref{p-mirror-2a}) with $\beta =p-1$ yields
$$ \left\|w\right\|_p^{p-1} \leq \left\|\tilde{w}\right\|_p^{p-1} + \left\|\tilde{w}-w\right\|_p^{p-1}. $$
Hence
$$ \left\|\left(\|\tilde{w}\|_p^{2-p} - \|w\|_p^{2-p}\right) \hat{w}\right\|_{p^*} \leq \left\|\tilde{w}\right\|_p^{p-1} \|\tilde{w}-w\|_p^{2-p} + \|\tilde{w}-w\|_p. $$
Combining this with (\ref{firsttermnablaPsi}) gives
$$ \big\|\nabla\Psi_p(\tilde{w})-\nabla\Psi_p(w)\big\|_{p^*} \leq \left(2 \|\tilde{w}\|_p\right)^{2-p} \left\|\tilde{w} - w\right\|_{p}^{p-1} + \left\|\tilde{w}\right\|_p^{p-1} \|\tilde{w}-w\|_p^{2-p} + \|\tilde{w}-w\|_p. $$
Putting this bound into (\ref{intermDPsi}), we obtain
$$ D_{\Psi_p}(\tilde{w}, w) \leq \left(2 \|\tilde{w}\|_p\right)^{2-p} \left\|\tilde{w} - w\right\|_{p}^{p} + \left\|\tilde{w}\right\|_p^{p-1} \|\tilde{w}-w\|_p^{3-p} + \|\tilde{w}-w\|_p^2. $$
Since $1< 3-p <2$, we have
$$ D_{\Psi_p}(\tilde{w}, w) \leq \left\{\begin{array}{ll} \left(\left(2 \|\tilde{w}\|_p\right)^{2-p} + \left\|\tilde{w}\right\|_p^{p-1}  + 1\right) \|\tilde{w}-w\|_p^2, & \hbox{when} \ \|\tilde{w}-w\|_p \geq 1, \\
\left(\left(2 \|\tilde{w}\|_p\right)^{2-p} + \left\|\tilde{w}\right\|_p^{p-1}  + 1\right) \|\tilde{w}-w\|_p^{\min\{p, 3-p\}}, & \hbox{when} \ \|\tilde{w}-w\|_p < 1. \end{array}\right.$$
Then our desired estimate (\ref{PsipBreg}) for $D_{\Psi_p}(\tilde{w}, w)$ follows.

Let $\tilde{w} \in \wcal$ and denote the constant $C_{\|\tilde{w}\|_p, p} = \left(\left(2 \|\tilde{w}\|_p\right)^{2-p} + \left\|\tilde{w}\right\|_p^{p-1}  + 1\right)^{-1}$. We know from (\ref{PsipBreg})
\begin{equation}\label{DPsilower}
 \|\tilde{w}-w\|_p^2+ \|\tilde{w}-w\|_p^{\min\{p, 3-p\}} \geq C_{\|\tilde{w}\|_p, p} D_{\Psi_p}(\tilde{w}, w).
\end{equation}

When $D_{\Psi_p}(\tilde{w}, w) \geq 1$, we have $\Omega_p \left(D_{\Psi_p}(\tilde{w}, w)\right)
= D_{\Psi_p}(\tilde{w}, w) +\frac{1}{\tau_p} -1 \leq D_{\Psi_p}(\tilde{w}, w)$ and
see from (\ref{DPsilower}) that either
$$ \|\tilde{w}-w\|_p^2 \geq 1  \Longrightarrow \|\tilde{w}-w\|_p^2 \geq \frac{1}{2} \left(\|\tilde{w}-w\|_p^2+ \|\tilde{w}-w\|_p^{\min\{p, 3-p\}}\right) \geq \frac{C_{\|\tilde{w}\|_p, p}}{2} \Omega_p \left(D_{\Psi_p}(\tilde{w}, w)\right) $$
or $\|\tilde{w}-w\|_p^2 < 1$ which implies
$$\|\tilde{w}-w\|_p^{\min\{p, 3-p\}} \geq \frac{C_{\|\tilde{w}\|_p, p}}{2} D_{\Psi_p}(\tilde{w}, w) \geq \frac{C_{\|\tilde{w}\|_p, p}}{2}$$
by our assumption $D_{\Psi_p}(\tilde{w}, w) \geq 1$, and thereby
\begin{align*}
 \|\tilde{w}-w\|_p^2 & = \|\tilde{w}-w\|_p^{\min\{p, 3-p\}} \|\tilde{w}-w\|_p^{2- \min\{p, 3-p\}}\\
 & \geq \left\{\frac{C_{\|\tilde{w}\|_p, p}}{2} D_{\Psi_p}(\tilde{w}, w)\right\} \left(\frac{C_{\|\tilde{w}\|_p, p}}{2}\right)^{\frac{2- \min\{p, 3-p\}}{\min\{p, 3-p\}}}.
\end{align*}
Hence
$$ \|\tilde{w}-w\|_p^2 \geq \min\left\{\frac{C_{\|\tilde{w}\|_p, p}}{2}, \left(\frac{C_{\|\tilde{w}\|_p, p}}{2}\right)^{\tau_p}\right\} \Omega_p \left(D_{\Psi_p}(\tilde{w}, w)\right). $$

When $D_{\Psi_p}(\tilde{w}, w) < 1$, we have $\Omega_p \left(D_{\Psi_p}(\tilde{w}, w)\right) = \frac{1}{\tau_p} \left(D_{\Psi_p}(\tilde{w}, w)\right)^{\tau_p}$. Again,
from (\ref{DPsilower}), we have either
\begin{align*}
\|\tilde{w}-w\|_p^2 < 1  &\Longrightarrow \|\tilde{w}-w\|_p^{\min\{p, 3-p\}} \geq \frac{C_{\|\tilde{w}\|_p, p}}{2} D_{\Psi_p}(\tilde{w}, w)\\
&\Longrightarrow \|\tilde{w}-w\|_p^{2} \geq \tau_p \left(\frac{C_{\|\tilde{w}\|_p, p}}{2}\right)^{\tau_p} \Omega_p \left(D_{\Psi_p}(\tilde{w}, w)\right)
\end{align*}
or $\|\tilde{w}-w\|_p^2 \geq 1$ which implies
$$\|\tilde{w}-w\|_p^{2} \geq \frac{C_{\|\tilde{w}\|_p, p}}{2} D_{\Psi_p}(\tilde{w}, w) \geq \frac{\tau_p C_{\|\tilde{w}\|_p, p}}{2} \Omega_p \left(D_{\Psi_p}(\tilde{w}, w)\right) $$
by our assumption $D_{\Psi_p}(\tilde{w}, w) < 1$. Therefore,
$$ \|\tilde{w}-w\|_p^2 \geq \min\left\{\tau_p \frac{C_{\|\tilde{w}\|_p, p}}{2}, \tau_p \left(\frac{C_{\|\tilde{w}\|_p, p}}{2}\right)^{\tau_p}\right\} \Omega_p \left(D_{\Psi_p}(\tilde{w}, w)\right). $$
Combining the above two cases and noting $\tau_p >1$, we see (\ref{Psipcondition}) holds.

The last statement follows immediately from the identity (\ref{Bregmansum}), the definition of $\sigma_F$-strong convexity, and (\ref{Psipcondition}). The proof is complete.
\end{proof}

\begin{proof}[Proof of Theorem \ref{thm:least-square}] Denote $\sup_{x\in\xcal}\|x\|_* =R>0$.
  The Hessian matrix of $f(\cdot, z) =\frac{1}{2}\left(\inn{\cdot, x} -y\right)^2$ for every $z$ is $\nabla_w^2[f(w,z)]=xx^\top$, from which we know that $f(\cdot,z)$ and $F$ are $R^2$-strongly smooth. Moreover, we have
  $$ \nabla F (w) = \ebb_Z [X X^\top w - X Y]  =\ccal_X w - \ebb_Z [XY].$$
  So we know from the positive definiteness of the covariance matrix $\ccal_X$ that the only minimizer $w^*$ is $w^* =w_\rho$.
  For any $w,\tilde{w}\in\wcal$, there holds
  \begin{align*}
     D_F(w,\tilde{w}) 
     & = \frac{1}{2}\ebb_Z\big[\big(\inn{w,X}-\inn{\tilde{w},X}+\inn{\tilde{w},X}-Y\big)^2\big]-\frac{1}{2}\ebb_Z\big[\big(\inn{\tilde{w},X}-Y\big)^2\big] - \inn{w-\tilde{w},\nabla F(\tilde{w})} \\
     & = \frac{1}{2}\ebb_Z\big[\big(\inn{w-\tilde{w},X}\big)^2\big]+\ebb_Z\big[\big\langle w-\tilde{w},\inn{\tilde{w},X}X-XY\big\rangle\big]-\inn{w-\tilde{w},\nabla F(\tilde{w})}\\
     & = \frac{1}{2}(w-\tilde{w})^\top\ccal_X(w-\tilde{w})\geq \frac{\lambda_{min}}{2} \|w-\tilde{w}\|_2^2,
  \end{align*}
where $\lambda_{min}>0$ is the smallest eigenvalue of the positive definite covariance matrix $\ccal_X$. But the norms $\|\cdot\|_2$ and $\|\cdot\|$ on $\rbb^d$ are equivalent. So there exist two positive numbers $b_1 \leq b_2$ such that $b_1 \|w\|^2 \leq \|w\|_2^2 \leq b_2 \|w\|^2$ for $w\in \rbb^d$. It follows that
$$D_F(w,\tilde{w}) \geq \frac{\lambda_{min} b_1}{2} \|w-\tilde{w}\|^2, \qquad \forall w,\tilde{w}\in \wcal.  $$
This verifies the $\lambda_{min} b_1$-strong convexity of $F$. So by Propositions \ref{prop:least-squares} and \ref{prop:p-divergence}, the conditions of Theorems \ref{thm:nece-suff}, \ref{thm:nece-suff-without-variance} and \ref{thm:ae-sufficient} are satisfied.
Moreover,
$$ \ebb_Z \left[\|\nabla_w [f(w, Z)]\|_*\right] = \ebb_Z \left[\|(Y -\inn{w,X}) X\|_*\right] = \ebb_Z \left[\left|Y -\inn{w,X}\right| \|X\|_*\right].  $$
So the assumption $\inf_{w\in\wcal} \ebb_Z \left[\|\nabla_w [f(w, Z)]\|_*\right]>0$ in Theorem \ref{thm:nece-suff}
is the same as the assumption $\inf_{w\in\wcal} \ebb_Z \left[\left|Y -\inn{w,X}\right| \|X\|_*\right]>0$ in Theorem \ref{thm:least-square},
and from Theorem \ref{thm:nece-suff} we know that if we replace $\|w_\rho - w_t\|^2$ by $D_{\Psi}(w_\rho, w_t)$, our statement (a) holds true and the constant $\sigma$ can be taken as $\sigma = \frac{2\lambda_{min} b_1}{L_\Psi}$ in the case of an $L_\Psi$-strongly smooth mirror map $\Psi$.
To get the statement for the norm square $\|w_\rho - w_t\|^2$, we notice first from the strong convexity of $\Psi$ that $\frac{\sigma_\Psi}{2} \|w_\rho - w_t\|^2 \leq D_{\Psi}(w_\rho, w_t)$.

When $\Psi$ is strongly smooth satisfying $D_{\Psi}(w_\rho, w_t) \leq \frac{L_\Psi}{2} \|w_\rho - w_t\|^2$, we know that our statement (a) holds true.
When $\Psi = \Psi_p$ for some $1< p \leq 2$, we use (\ref{Psipcondition}) with $\tilde{w}=w_\rho$ and Jensen's inequality to get from the convexity of $\Omega$
$$  \ebb_{z_1, \ldots, z_{t-1}} [\|w_\rho -w_t\|^2] \geq B'_{p} \Omega_p \left(\ebb_{z_1, \ldots, z_{t-1}}[D_{\Psi_p}(w_\rho, w_t)]\right), $$
where $B'_{p}$ is a constant depending on $p, \|w_\rho\|,$ and a constant $c_p$ such that $c_p \|w\|_p\leq \|w\|$ holds for every $w\in \wcal$. Combining this relation with the explicit formula (\ref{Omegap}) for $\Omega_p$, we know that $\lim_{t\to\infty} \ebb_{z_1, \ldots, z_{t-1}} [\|w_\rho -w_t\|^2]=0$ implies $\lim_{t\to\infty}\ebb_{z_1, \ldots, z_{t-1}}[D_{\Psi_p}(w_\rho, w_t)]=0$. Hence our statement (a) also holds true for $\Psi = \Psi_p$.

Note that the assumption $\ebb_Z \left[\|\nabla_w [f(w^*, Z)]\|_*\right]=0$ in our statement (b) of Theorem \ref{thm:nece-suff-without-variance}
is the same as the the assumption $\ebb_Z \left[\left|Y -\inn{w_\rho,X}\right| \|X\|_*\right]=0$ in Theorem \ref{thm:least-square}.
So our statement (b) can be proved from Theorem \ref{thm:nece-suff-without-variance}
by the same argument for dealing with the norm square $\|w_\rho - w_t\|^2$ from $D_{\Psi}(w_\rho, w_t)$ as we did for our statement (a).

Our statement (c) follows from Theorem \ref{thm:ae-sufficient} and the strong convexity of $\Psi$.
The proof of Theorem \ref{thm:least-square} is complete.
\end{proof}

\begin{proof}[Proof of Theorem \ref{thm:p-norm-convergence}]
Recall that for the regularizer $r$ given by $r(w) =\lambda\|w\|_2^2$, there holds $D_{r}(\tilde{w}, w)=\lambda \|\tilde{w}-w\|_2^2$ for $\tilde{w}, w\in\wcal$.
So we know that $F$ is $2\lambda$-strongly convex for every $z\in \zcal$.

For the Bregman distance induced by the loss function
$$D_{\phi(\inn{\cdot, x}, y)}(\tilde{w}, w)=\phi(\inn{\tilde{w},x}, y) - \phi(\inn{w,x}, y) -\inn{\tilde{w}-w, \phi'(\inn{w,x}, y) x}, $$
we apply the mean value theorem to find
$$ \phi(\inn{\tilde{w},x}, y) - \phi(\inn{w,x}, y) =\phi'(\xi, y) \left(\inn{\tilde{w},x} - \inn{w,x}\right) = \inn{\tilde{w}-w, \phi'(\xi, y) x}, $$
where $\xi$ is a number between $\inn{\tilde{w},x}$ and $\inn{w,x}$. We can write
$$\xi = (1-\theta) \inn{\tilde{w},x} + \theta \inn{w,x} = \inn{(1-\theta) \tilde{w}+ \theta w,x} $$
for some $\theta \in (0, 1)$. It follows that
$$ D_{\phi(\inn{\cdot, x}, y)}(\tilde{w}, w)=\inn{\tilde{w}-w, \left(\phi'(\inn{(1-\theta) \tilde{w}+ \theta w,x}, y) - \phi'(\inn{w,x}, y)\right) x} $$
and
$$ D_{\phi(\inn{\cdot, x}, y)}(\tilde{w}, w) \leq \|\tilde{w}-w\| \|x\|_* \left|\phi'(\inn{(1-\theta) \tilde{w}+ \theta w,x}, y) - \phi'(\inn{w,x}, y)\right|. $$
Then we apply the Lipschitz condition (\ref{Lipphi}) and obtain
$$ D_{\phi(\inn{\cdot, x}, y)}(\tilde{w}, w) \leq \|\tilde{w}-w\| \|x\|_* \ell_\phi \left|\inn{(1-\theta) \tilde{w}+ \theta w,x} -\inn{w,x}\right|
\leq \|\tilde{w}-w\|^2 \|x\|_*^2 \ell_\phi. $$
If we denote $\sup_{x\in\xcal}\|x\|_* =R>0$, then we have
$$ D_{\phi(\inn{\cdot, x}, y)}(\tilde{w}, w) \leq \ell_\phi R^2 \|\tilde{w}-w\|^2, \qquad \forall \tilde{w}, w \in \wcal. $$
Therefore, $f(\cdot, z)$ is $2(\ell_\phi R^2 + \lambda)$-strongly smooth for every $z\in \zcal$, and the statements on the strong smoothness of $F$ follows.
Our desired statement on the convergence follows from Theorems \ref{thm:nece-suff}, \ref{thm:nece-suff-without-variance} and \ref{thm:ae-sufficient}, as we have done in the proof of Theorem \ref{thm:least-square}. The proof of Theorem \ref{thm:p-norm-convergence} is complete.
\end{proof}

\section*{Appendix}

This appendix provides the proofs of the co-coercivity of gradients stated in Lemma \ref{lem:co-coercivity} and Proposition \ref{prop:p-divergence-nonsmooth} together with a remark on variances involving stochastic gradients.

To prove Lemma \ref{lem:co-coercivity}, we need the following lemma on the Fenchel-conjugate of some norm power functions which is of independent interest.

\begin{lemma}\label{lem:fenchel-conjugate-p-norm-calculation}
  Let $\kappa>1$. The Fenchel-conjugate of $f=\frac{1}{\kappa}\|\cdot\|^\kappa$ is given by $f^*(v)=\frac{\kappa-1}{\kappa}\|v\|_*^{\frac{\kappa}{\kappa-1}}$.
\end{lemma}
\begin{proof}
  According to Young's inequality $ab \leq \frac{1}{\kappa} a^{\kappa} + \frac{\kappa -1}{\kappa} a^{\frac{\kappa}{\kappa -1}}$, we have for $v\in \wcal^*$,
  \begin{align*}
    f^*(v) &= \sup_{w\in\wcal}\big[\inn{w,v}-\frac{1}{\kappa}\|w\|^\kappa\big] \leq \sup_{w\in\wcal}\big[\|w\|\|v\|_*-\frac{1}{\kappa}\|w\|^\kappa\big]\\
    & \leq \sup_{w\in\wcal}\Big[\frac{1}{\kappa}\|w\|^\kappa+\frac{\kappa-1}{\kappa} \|v\|_*^{\frac{\kappa}{\kappa-1}}-\frac{1}{\kappa}\|w\|^\kappa\Big]\\
    &= \frac{\kappa-1}{\kappa}\|v\|_*^{\frac{\kappa}{\kappa-1}}.
  \end{align*}
Since $\wcal = \wcal^{**}$, for $v\in \wcal^*$, there exists some $w\in\wcal= \wcal^{**}$ such that $\inn{w,v}=\|v\|_*$ and $\|w\|=1$.
Taking the vector $\|v\|_*^{\frac{1}{\kappa-1}} w$ in the definition of $f^*$ gives
$$
    f^*(v)\geq\inn{\|v\|_*^{\frac{1}{\kappa-1}} w,v}-\frac{1}{\kappa}\|w\|^\kappa\|v\|_*^{\frac{\kappa}{\kappa-1}}=\|v\|_*^{\frac{1}{\kappa-1}}\|v\|_*-\frac{1}{\kappa}\|v\|_*^{\frac{\kappa}{\kappa-1}}=\frac{\kappa-1}{\kappa}\|v\|_*^{\frac{\kappa}{\kappa-1}}.
$$
  Combining the above two inequalities yields the stated result.
\end{proof}
\begin{proof}[Proof of Lemma \ref{lem:co-coercivity}]
  We use some ideas from \citep{ying2015unregularized}.
  Fix a $w\in\wcal$. Define $h:\wcal\to\rbb$ by $h(\bar{w})=g(\bar{w})-\inn{\bar{w},\nabla g(w)}$. It is clear that $h$ satisfies the condition
  $$
    D_h(\bar{w},\tilde{w})=D_g(\bar{w},\tilde{w})\leq \frac{L}{1+\alpha}\|\bar{w}-\tilde{w}\|^{1+\alpha},\quad\forall \bar{w},\tilde{w}\in\wcal.
  $$
  Since $h$ is convex and $\nabla h(w)=0$, we know that $h$ attains its minimum at $w$. So for $\tilde{w} \in \wcal$, we have
  \begin{align*}
    h(w) & = \min_{\bar{w}\in\wcal}h(\bar{w})\leq \min_{\bar{w}\in\wcal}\Big[h(\tilde{w})+\inn{\bar{w}-\tilde{w},\nabla h(\tilde{w})}+\frac{L}{1+\alpha}\|\tilde{w}-\bar{w}\|^{\alpha+1}\Big] \\
     & = h(\tilde{w})-L\max_{\bar{w}\in\wcal}\Big[\inn{\tilde{w}-\bar{w},L^{-1}\nabla h(\tilde{w})}-\frac{1}{1+\alpha}\|\tilde{w}-\bar{w}\|^{\alpha+1}\Big] \\
     & = h(\tilde{w})-L\max_{\bar{w}\in\wcal}\Big[\inn{\bar{w},L^{-1}\nabla h(\tilde{w})}-\frac{1}{1+\alpha}\|\bar{w}\|^{\alpha+1}\Big].
  \end{align*}
  According to the definition of Fenchel-conjugate and Lemma \ref{lem:fenchel-conjugate-p-norm-calculation} with $\kappa=\alpha+1$, we know
  \begin{align*}
  \max_{\bar{w}\in\wcal}\Big[\inn{\bar{w},L^{-1}\nabla h(\tilde{w})}-\frac{1}{1+\alpha}\|\bar{w}\|^{\alpha+1}\Big]
  &=\Big(\frac{1}{1+\alpha}\|\cdot\|^{\alpha+1}\Big)^*(L^{-1}\nabla h(\tilde{w}))\\
  &=\frac{\alpha}{1+\alpha}\big\|L^{-1}\nabla h(\tilde{w})\big\|_*^{\frac{1+\alpha}{\alpha}}.
  \end{align*}
  Combining the above discussions implies
  $$
    h(w)\leq h(\tilde{w})-\frac{L^{-\frac{1}{\alpha}}\alpha}{1+\alpha}\big\|\nabla h(\tilde{w})\big\|_*^{\frac{1+\alpha}{\alpha}},\qquad\forall \tilde{w}\in\wcal.
  $$
  The above inequality can be equivalently written as
  $$
   g(\tilde{w})\geq g(w)+\inn{\tilde{w}-w,\nabla g(w)}+\frac{L^{-\frac{1}{\alpha}}\alpha}{1+\alpha}\|\nabla g(\tilde{w})-\nabla g(w)\|_*^{\frac{1+\alpha}{\alpha}}.
  $$
Switching $w$ and $\tilde{w}$ also shows
  $$
   g(w)\geq g(\tilde{w})+\inn{w-\tilde{w},\nabla g(\tilde{w})}+\frac{L^{-\frac{1}{\alpha}}\alpha}{1+\alpha}\|\nabla g(w)-\nabla g(\tilde{w})\|_*^{\frac{1+\alpha}{\alpha}}.
  $$
  Summing up the above two inequalities gives the stated inequality \eqref{sufficient-condition-b} and completes the proof.
\end{proof}

Now we turn to the proof of Proposition \ref{prop:p-divergence-nonsmooth}.

\begin{proof}[Proof of Proposition \ref{prop:p-divergence-nonsmooth}]
Recall the dual number $p^* =\frac{p}{p-1} >2$ of $p$ given in the proof of Proposition \ref{prop:p-divergence} satisfying $\frac{1}{p} + \frac{1}{p^*} =1$. Take the norm $\|\cdot\|=\|\cdot\|_p$.

Suppose to the contrary that $\Psi_p$ is $L$-strong smooth for some $L>0$. Then we know from the inequality (\ref{gradcont}) derived in the proof of Proposition \ref{prop:least-squares} that
\begin{equation}\label{gradcontcontra}
  \|\nabla\Psi_p (w)-\nabla\Psi_p (\tilde{w})\|_*\leq L \|w-\tilde{w}\|,\qquad\forall w,\tilde{w}\in\wcal.
\end{equation}

Let $a\geq1$ and define two vectors $w,\tilde{w}\in\rbb^d$ as
  $$
  w=\left\{\begin{array}{ll} (a+1,a-1,\ldots,a+1,a-1), & \hbox{if $d$ is even}, \\
  (a+1,a-1,\ldots,a+1,a-1, a), & \hbox{if $d$ is odd}, \end{array}\right. $$
  and
  $$ \tilde{w}=\left\{\begin{array}{ll} (a-1,a+1,\ldots,a-1,a+1), & \hbox{if $d$ is even}, \\
  (a-1,a+1,\ldots,a-1,a+1, a), & \hbox{if $d$ is odd}. \end{array}\right.
  $$
  By the elementary inequality $(a+1)^p+(a-1)^p\geq 2a^p$, we find
  $$
    \|w\|_p=\|\tilde{w}\|_p=\left\{\begin{array}{ll} \left[\frac{d}{2}(a+1)^p+\frac{d}{2}(a-1)^p\right]^{1\over p}\geq d^{1\over p} a, & \hbox{if $d$ is even}, \\
  \left[\frac{d-1}{2}(a+1)^p+\frac{d-1}{2}(a-1)^p + a^p\right]^{1\over p} \geq d^{1\over p}a, & \hbox{if $d$ is odd}.  \end{array}\right.
  $$
Combining this with the expression of $\nabla\Psi_p$ given in \eqref{gradient-p-divergence} yields
\begin{align*}
\|\nabla\Psi_p(w)-\nabla\Psi_p(\tilde{w})\|_* &= \|w\|_p^{2-p}\big\|\big(|w(j)|^{p-1}-|\tilde{w}(j)|^{p-1}\big)_{j=1}^d\big\|_* \\
&\geq \|w\|_p^{2-p} [(a+1)^{p-1}-(a-1)^{p-1}] (d-1)^{\frac{1}{p^*}} \\
     & \geq (d-1)^{\frac{1}{p}}a^{2-p}[(a+1)^{p-1}-(a-1)^{p-1}].
\end{align*}
But
$$ \|w-\tilde{w}\| =\left\{\begin{array}{ll}  2 d^{1/p}, & \hbox{if $d$ is even}, \\
 2 (d-1)^{1/p} < 2 d^{1/p}, & \hbox{if $d$ is odd}. \end{array}\right. $$
 It follows that
$$ \|\nabla\Psi_p(w)-\nabla\Psi_p(\tilde{w})\|_* \geq \frac{1}{2} \left(\frac{d-1}{d}\right)^{\frac{1}{p}} a^{2-p}[(a+1)^{p-1}-(a-1)^{p-1}]\|w-\tilde{w}\|.
$$
Since $d\geq 2$, we have $\frac{d-1}{d} \geq \frac{1}{2}$. Therefore we apply the inequality (\ref{gradcontcontra}) to obtain
$$ L \|w-\tilde{w}\| \geq \frac{1}{4} a^{2-p}[(a+1)^{p-1}-(a-1)^{p-1}]\|w-\tilde{w}\|. $$
This is a contradiction to the limit $\lim_{a\to\infty}a^{2-p}[(a+1)^{p-1}-(a-1)^{p-1}]=\infty$. So $\Psi_p$ is not strong smooth.  The proof of Proposition \ref{prop:p-divergence-nonsmooth} is complete.
\end{proof}

At the end, we give the following remark on the conditions on the variances.

\begin{proposition}\label{prop:variance}
    If $F$ is Fr\'{e}chet differentiable, then the following two statements hold.
  \begin{enumerate}[(a)]
    \item If there exists a $w^*\in\wcal$ with $\ebb_Z[\|\nabla_w[f(w^*,Z)]\|_*]=0$, then we have $\ebb_Z[\|\nabla_w[f(w^*,Z)]-\nabla F(w^*)\|_*^2]=0$.
    \item If $\inf_{w\in\wcal} \ebb_Z[\|\nabla_w[f(w,Z)]\|_*]>0$, then we have $\ebb_Z[\|\nabla_w[f(w^*,Z)]-\nabla F(w^*)\|_*^2]>0$ for any minimizer $w^*$ of $F$.
  \end{enumerate}
\end{proposition}
\begin{proof}
For the statement (a), the condition $\ebb_Z[\|\nabla_w[f(w^*,Z)]\|_*]=0$ amounts to saying that $\nabla_w[f(w^*,Z)]=0$ holds almost surely, from which it follows that $\nabla F(w^*)=0$ and therefore $\ebb_Z[\|\nabla_w[f(w^*,Z)]-\nabla F(w^*)\|_*^2]=0$.

The statement (b) follows from the optimality condition $\nabla F(w^*)=0$ and the Schwarz inequality $\ebb_Z[\|\nabla_w[f(w^*,Z)]\|_*] \leq \left\{\ebb_Z[\|\nabla_w[f(w^*,Z)]\|_*^2]\right\}^{1/2}$.
\end{proof}

\setlength{\bibsep}{0.03cm}
\bibliographystyle{abbrvnat}
\small

\end{document}